\documentclass[10pt,journal,compsoc]{IEEEtran}

\ifCLASSOPTIONcompsoc
  \usepackage[nocompress]{cite}
\else
  \usepackage{cite}
\fi

\usepackage{amsmath,amsfonts,bm}

\def\eqref#1{equation~\ref{#1}}

\def\1{\bm{1}}

\DeclareMathAlphabet{\mathsfit}{\encodingdefault}{\sfdefault}{m}{sl}
\SetMathAlphabet{\mathsfit}{bold}{\encodingdefault}{\sfdefault}{bx}{n}

\DeclareMathOperator*{\argmin}{arg\,min}

\usepackage{hyperref}
\usepackage{url}

\usepackage{microtype}
\usepackage{graphicx}
\usepackage{subcaption}
\usepackage{booktabs}
\usepackage{hyperref}

\usepackage{amsmath}
\usepackage{amssymb}
\usepackage{mathtools}
\usepackage{amsthm}
\usepackage{enumitem}
\usepackage{pifont}

\usepackage[table]{xcolor}

\usepackage[capitalize,noabbrev]{cleveref}
\usepackage{multirow}

\theoremstyle{plain}
\newtheorem{theorem}{Theorem}[section]
\newtheorem{proposition}[theorem]{Proposition}

\theoremstyle{definition}
\newtheorem{definition}[theorem]{Definition}
\newtheorem{assumption}[theorem]{Assumption}
\theoremstyle{remark}
\newtheorem{remark}[theorem]{Remark}

\usepackage[textsize=tiny]{todonotes}

\begin{document}
%
\title{\textbf{LA}tent \textbf{P}hase \textbf{I}nference from \textbf{S}hort\\ time sequences using \textbf{SH}allow \textbf{RE}current \textbf{D}ecoders (LAPIS-SHRED)}

\author{Yuxuan~Bao,
        Xingyue~Zhang,
        and~J.~Nathan~Kutz
\IEEEcompsocitemizethanks{
\IEEEcompsocthanksitem Y. Bao is with the Department of Applied Mathematics, University of Washington.\protect
\IEEEcompsocthanksitem X. Zhang is with the School of Environmental and Forest Sciences, University of Washington.\protect
\IEEEcompsocthanksitem J. N. Kutz is with Autodesk Research, London, UK. \protect
\IEEEcompsocthanksitem Corresponding authors: baoyx@uw.edu, kutz@uw.edu.}
}

\markboth{}%
{Shell \MakeLowercase{\textit{et al.}}: Bare Advanced Demo of IEEEtran.cls for IEEE Computer Society Journals}

\IEEEtitleabstractindextext{%
\begin{abstract}
Reconstructing full spatio-temporal dynamics from sparse observations in both space and time remains a central challenge in complex systems, as measurements can be spatially incomplete and can be also limited to narrow temporal windows.  Yet approximating the complete spatio-temporal trajectory is essential for mechanistic insight and understanding, model calibration, and operational decision-making. We introduce LAPIS-SHRED (\textbf{LA}tent \textbf{P}hase \textbf{I}nference from \textbf{S}hort time sequence using \textbf{SH}allow \textbf{RE}current \textbf{D}ecoders), a modular architecture that reconstructs and/or forecasts complete spatiotemporal dynamics from sparse sensor observations confined to short temporal windows. LAPIS-SHRED operates through a three-stage pipeline: (i) a SHRED model is pre-trained \emph{entirely on simulation data} to map sensor time-histories into a structured latent space, (ii) a temporal sequence model, trained on simulation-derived latent trajectories, learns to propagate latent states forward or backward in time to span unobserved temporal regions from short observational time windows, and (iii) at deployment, only a short observation window of hyper-sparse sensor measurements from the true system is provided, from which the frozen SHRED model and the temporal model jointly reconstruct or forecast the complete spatiotemporal trajectory. The framework supports bidirectional (in time) inference, inherits data assimilation and multiscale reconstruction capabilities from its modular structure, and accommodates extreme observational constraints including single-frame terminal inputs. We evaluate LAPIS-SHRED on six experiments spanning complex spatio-temporal physics: turbulent flows, multiscale propulsion physics, volatile combustion transients, and satellite-derived environmental fields. Across all settings, LAPIS-SHRED achieves consistently low NRMSE using as few as 3 sensors and observation windows as short as 7\% or less of the temporal domain, closely approaching baselines that observe the entire temporal record. These results establish LAPIS-SHRED as a latent phase inference framework applicable to a broad class of dynamic systems where simulation training data is available but dense temporal observations at deployment are not---highlighting a lightweight, modular architecture suited to rapid post-event reconstruction from terminal snapshots, resource-constrained forecasting, and operational settings where observation windows are constrained by physical or logistical limitations.
\end{abstract}

}

\maketitle

\IEEEdisplaynontitleabstractindextext

\IEEEpeerreviewmaketitle

\section{Introduction}
\label{sec:intro}

Reconstructing the complex spatiotemporal evolution patterns in physical systems from limited observational data is a pervasive challenge across sciences and engineering~\cite{kutz2025accelerating,evensen2009data,tarantola2005inverse}. Physical systems---whether governed by continuum mechanics, stochastic kinetics, or coupled multiscale processes---frequently evolve through transient regimes that are inaccessible to direct measurement, either because the relevant scales are too fast, the environment too hazardous, or the instrumentation too sparse to capture the full state~\cite{brunton2024promising}. In many domains, observations are not merely spatially sparse but also \emph{temporally limited}: sensor measurements may be available only during a brief terminal window (e.g., post-event deposits in geophysics~\cite{iverson1997physics,hungr2014varnes}, post-mortem structural damage~\cite{farrar2012structural}), or only during a short initial window (e.g., early-phase combustion monitoring~\cite{raman2023nonidealities} or brief satellite revisit periods~\cite{weiss2020remote}). Recovering the complete temporal trajectory from such short observation windows---while using only a small number of spatial sensors---constitutes a severely data-constrained and ill-posed inference problem.

\begin{figure*}[t]
\centering
\includegraphics[width=\textwidth]{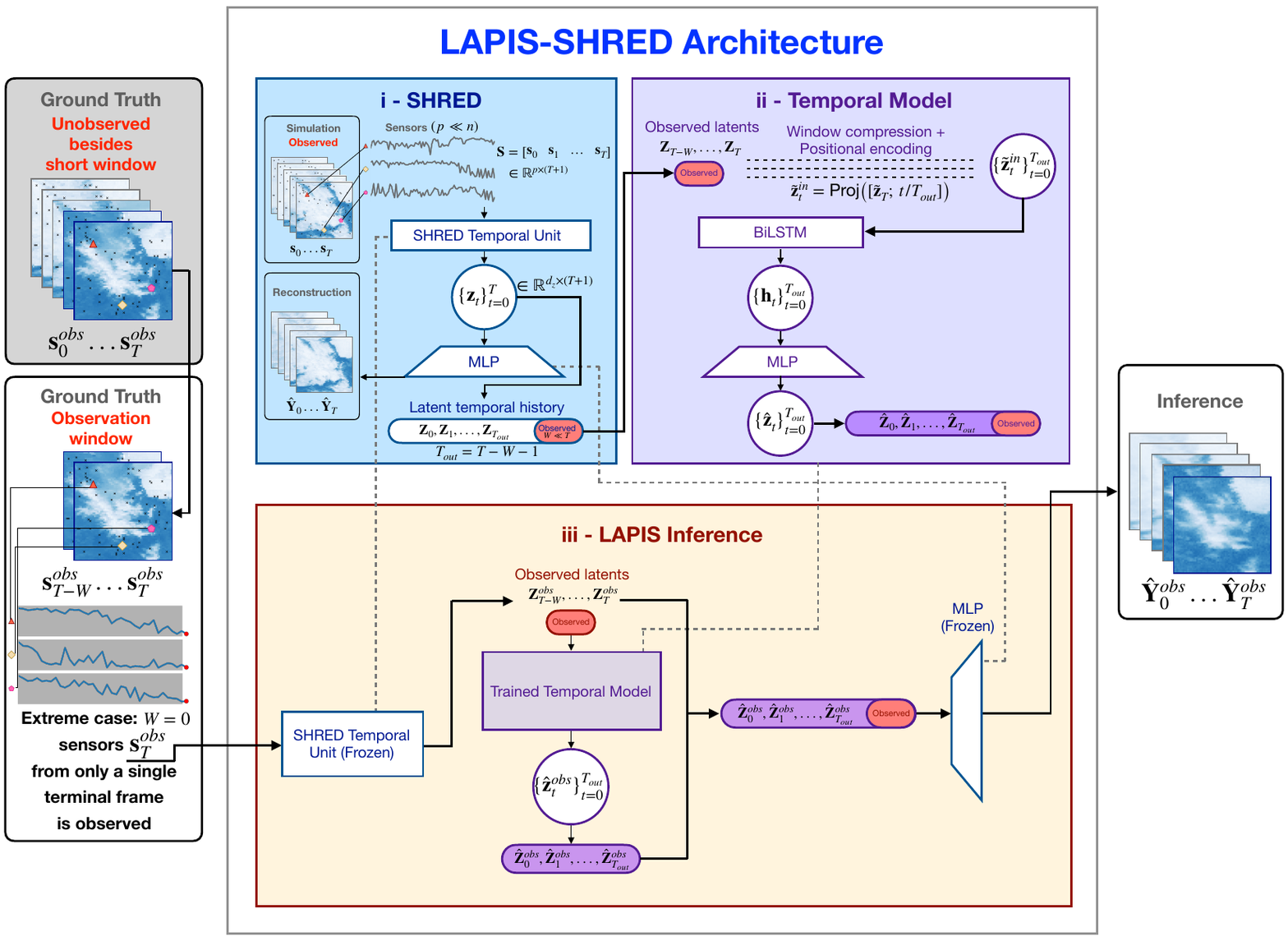}
\caption{Overview of the LAPIS-SHRED architecture. \textbf{Stage i (SHRED):} Sparse sensor observations $\mathbf{S} = [\mathbf{s}_0, \ldots, \mathbf{s}_T] \in \mathbb{R}^{p \times (T+1)}$ from simulation ensembles are processed via a pre-trained SHRED model into latent representations $\mathbf{z}_0, \ldots, \mathbf{z}_T$, which are decoded to reconstruct the full spatiotemporal trajectory. A short window of latent codes---from either the initial or terminal phase---is extracted for temporal model training. \textbf{Stage ii (Temporal Model):} Trained on simulation-derived latent trajectories, a BiLSTM architecture learns to map short-window latent codes (augmented with positional encodings) to the unobserved portion of the latent trajectory. \textbf{Stage iii (LAPIS Inference):} At test time, only the short observation window from the ground truth system (reality) is available. The frozen SHRED model extracts the observed-window latent codes, the temporal model infers the unobserved latent trajectory, and the frozen decoder reconstructs the complete spatiotemporal evolution. The architecture shown depicts backward inference from a terminal time window; forward inference from an initial time window proceeds analogously.}
\label{fig:lapis_architecture}
\end{figure*}

Existing approaches span a broad methodological landscape but are not designed for the proposed short-window observation regime. Classical data assimilation methods~\cite{evensen2009data,tarantola2005inverse,asch2016data} require dense or sequential temporal observations and struggle with highly nonlinear dynamics. Physics-informed neural networks (PINNs)~\cite{raissi2019physics,karniadakis2021physics,fan2025physics} embed governing equations into training losses but require explicit knowledge of the {\em partial differential equation} (PDE) form and sufficient constraints (like initial/boundary conditions and scattered observations) to anchor the solution, with performance degrading under extreme spatial and temporal sparsity. Neural and Fourier neural operator (FNO) frameworks~\cite{li2020fourier,lu2021learning,kovachki2023neural} learn mappings between function spaces and can generalize across families of PDEs and parameter regimes, yet typically require large paired training datasets spanning the full temporal domain, and face challenges with strongly multiscale dynamics and inverse problems where only partial observations are available.  Traditional reduced-order models (ROMs) via proper orthogonal decomposition (POD)~\cite{benner2015survey} or dynamic mode decomposition (DMD)~\cite{kutz2016dynamic} provide efficient low-dimensional representations, with the former method lacking assimilation capabilities and the latter requiring full time series observations.  Both POD and DMD based ROMs are not suited for the hyper-sparse observations (as low as three sensors) proposed in this work.  Additionally, methods from PINNs to FNOs have difficulty accurately modeling spatio-temporal dynamics~\cite{wyder2025common}.  As an alternative, the SHRED architecture ~\cite{williams2024sensing,gao2025sparse,tomasetto2025reduced,bao2025data,yermakov2025t} and its multiscale extensions~\cite{bao2026cheap2rich,zhang2026sendai} have shown that hyper-sparse sensor time-histories processed through temporal units and shallow decoders are effective ROMs for modeling high-dimensional spatial states by leveraging Takens' embedding theorem~\cite{takens2006detecting}.  Thus unlike most existing methods that assume access to sparse, but not hyper-sparse sensor observations spanning the full temporal extent of the trajectory, SHRED can be modified to accommodate the short temporal window observations of interest considered here.

In this work, we present LAPIS-SHRED (\textbf{LA}tent \textbf{P}hase \textbf{I}nference from \textbf{S}hort time sequence using \textbf{S}hallow \textbf{RE}current \textbf{D}ecoders (LAPIS-SHRED)), a modular deep learning architecture that reconstructs and/or forecasts full-state spatiotemporal dynamics from sparse sensor measurements confined to a short temporal window (See Fig.~\ref{fig:lapis_architecture}). The key insight is that a pre-trained SHRED model induces a structured latent space in which temporal dynamics are encoded as low-dimensional trajectories~\cite{gao2025sparse}. LAPIS-SHRED exploits this structure through a three-stage pipeline: (i)~a SHRED model is trained exclusively on simulation ensembles to learn a mapping from sparse sensor time-histories to a structured latent space; (ii)~a \emph{temporal dynamics model}---either sequence-to-sequence or autoregressive---is trained on the resulting simulation-derived latent trajectories to propagate latent states forward or backward in time. In the deployment phase: (iii)~only a short observation window of hyper-sparse sensor measurements from the true system is provided; the frozen SHRED model encodes these observations into the learned latent space, the temporal model infers the unobserved latent trajectory, and the frozen decoder reconstructs the complete spatiotemporal evolution. No ground truth beyond this short, sparse observation window is required at deployment. By operating entirely in the learned latent space. LAPIS-SHRED inherits SHRED's spatial reconstruction capabilities while adding temporal extrapolation capacity that the latent space alone does not possess. Figure~\ref{fig:lapis_architecture} illustrates the complete pipeline.
The principal contributions of this work are as follows:\\

\begin{enumerate}[nosep]
    \item \textit{Short-window temporal inference.} LAPIS-SHRED reconstructs full spatiotemporal trajectories from observation windows comprised of only 7--20\% of the temporal data, using as few as 3 spatial sensors, achieving NRMSE consistently below 5\% across diverse dynamical regimes---performance approaching that of SHRED baselines which have access to the complete temporal record.
    \item \textit{Bidirectional inference with structured modularity.} The framework supports both backward reconstruction from terminal windows and forward prediction from initial windows, and interfaces SHRED---frame-by-frame, Seq2Seq, or multiscale---without modification to the temporal inference stages.
    \item \textit{Single-snapshot encoding strategies.} In an extreme sparsity case when the accessible data are reduced to a single temporal snapshot at the terminal state, we introduce a padding mechanism that enables latent-space encoding from a single terminal observation frame, and a Seq2Seq temporal model with positional encoding that produces full latent trajectories in a single forward pass.\\
\end{enumerate}

\noindent We evaluate LAPIS-SHRED on six challenging multiscale experiments spanning proto-typical chaotic PDEs (2D Kuramoto--Sivashinsky, 2D Kolmogorov flow), turbulent flow (2D von Karman vortex street), multiscale propulsion physics (high-fidelity rotating detonation engine), volatile combustion transients (1D RDE ignition), and satellite-derived environmental fields (MODIS snow cover). We additionally benchmark against SHRED-ROM~\cite{tomasetto2025reduced} and discuss the applicability of neural operator methods including FNO~\cite{li2020fourier} and DeepONet~\cite{lu2021learning} in Section~\ref{sec:baseline_comparisons}.

\section{Preliminaries:  The SHRED Architecture}


\label{sec:shred_family}

The SHallow REcurrent Decoder (SHRED) architecture~\cite{williams2024sensing} addresses the task of reconstructing high-dimensional spatiotemporal states from sparse (hyper-sparse), pointwise sensor measurements. SHRED rests on three foundational principles: (i)~separation of variables, which decouples the spatial reconstruction from the temporal encoding; (ii)~Takens' embedding theorem~\cite{takens2006detecting} which guarantees that the time-history at a generic sensor location provides a diffeomorphic representation of the underlying attractor; and (iii)~a decoding-only strategy that avoids the ill-conditioning inherent in learning inverse operators through standard autoencoder pairs. The architecture consists of a temporal unit---typically a multi-layer LSTM~\cite{hochreiter1997long}---that maps a lagged window of sparse sensor observations ${\bf S}$ (See Fig.~\ref{fig:lapis_architecture}) to a low-dimensional latent code $\mathbf{z}_t \in \mathbb{R}^{d_z}$, followed by a shallow feedforward decoder that maps $\mathbf{z}_t$ to the full spatial state $\hat{\mathbf{Y}}_t \in \mathbb{R}^n$.  Training is performed on simulation data: sparse sensor measurements are extracted from the simulation fields via the sampling operator, and the temporal unit--decoder pair is optimized end-to-end to minimize a reconstruction loss---computed either over the full spatial state or only at the sensor locations, depending on the available supervision.

Several works have extended the SHRED framework to address specific challenges in scientific modeling. SINDy-SHRED~\cite{gao2025sparse} regularizes the latent space by embedding a Sparse Identification of Nonlinear Dynamics (SINDy)~\cite{brunton2016discovering} loss on the latent trajectories, encouraging the learned latent dynamics to conform to parsimonious nonlinear functional forms. This regularization improves noise robustness and interpretability. SHRED-ROM~\cite{tomasetto2025reduced} recasts the reconstruction in a reduced-order setting by predicting low-dimensional basis coefficients (e.g., POD) from sparse sensor time histories. T-SHRED~\cite{yermakov2025t} introduces a transformer-based model in place of the LSTM, capturing temporal dependencies through attention mechanisms.

Beyond reconstruction, several architectures address the simulation-to-reality (sim2real) gap through data assimilation. DA-SHRED~\cite{bao2025data} trains the SHRED model on simulation data and then updates the latent representations using real sensor observations, closing the discrepancy between the simulated and true system states. This is achieved on observed sensor signals without requiring access to the true full-state field. SINDy-based regression in the latent space is also incorporated to identify explicit functional forms corresponding to missing physics. Cheap2Rich~\cite{bao2026cheap2rich} extends this to multi-fidelity paradigm by decomposing the state reconstruction into a low-frequency (LF) pathway trained on a computationally inexpensive reduced-order model and a high-frequency (HF) pathway that captures unmodeled fine-scale physics. SENDAI~\cite{zhang2026sendai} further addresses multiscale systems through a hierarchical architecture with spectrally sparse corrections, enabling robust reconstruction under domain shift between simulation and reality. 




\subsection{Mathematical Formulation of SHRED}
\label{sec:shred_backbone}

The SHRED architecture (Section~\ref{sec:shred_family}) maps sparse sensor histories to full spatial states through a temporal unit and a shallow decoder. It can be deployed in two operational modes within LAPIS-SHRED.

\textbf{In the \emph{frame-by-frame} mode}, standard SHRED employs time-delay embedding with a sliding window of $\ell$ lagged observations to reconstruct individual spatial frames:
\begin{equation}
\hat{\mathbf{Y}}_t = \mathcal{D}_{\text{SHRED}}\!\left(\mathcal{E}_{\text{SHRED}}(\mathbf{s}_{t-\ell+1:t};\, \phi_{\text{temp}});\, \phi_{\text{dec}}\right),
\label{eq:frame_shred}
\end{equation}
where the temporal unit $\mathcal{E}_{\text{SHRED}}$ is a multi-layer LSTM that processes the lagged sensor window and produces a latent code $\mathbf{z}_t \in \mathbb{R}^{d_z}$, and the decoder $\mathcal{D}_{\text{SHRED}}$ is a multi-layer perceptron mapping the latent code to the spatial state. This frame-by-frame approach excels in chaotic systems where the maximal Lyapunov exponent $\lambda_{\max} > 0$ implies exponential sensitivity to initial conditions~\cite{lorenz2017deterministic, ott2002chaos}, and reconstruction benefits from \emph{localized} temporal windows that capture instantaneous attractor geometry~\cite{pathak2018model, vlachas2018data}.

\textbf{In the \emph{sequence-to-sequence} (Seq2Seq) mode}, the temporal unit processes the full sensor sequence and produces a latent trajectory:
\begin{equation}
\mathbf{Z} = \mathcal{E}_{\text{SHRED}}(\mathbf{S};\, \phi_{\text{temp}}) \in \mathbb{R}^{(T+1) \times d_z},
\label{eq:seq2seq_encoder}
\end{equation}
where $\mathbf{S} = [\mathbf{s}_0, \ldots, \mathbf{s}_{T}]^\top \in \mathbb{R}^{(T+1) \times p}$ is the sensor time-series and $d_z = 2d_h$ arises from the concatenation of forward and backward hidden states in a bidirectional LSTM. The decoder then maps the full latent trajectory to the full spatiotemporal state:
\begin{equation}
\hat{\mathbf{X}} = \mathcal{D}_{\text{SHRED}}(\mathbf{Z};\, \phi_{\text{dec}}) \in \mathbb{R}^{(T+1) \times n}.
\label{eq:seq2seq_decoder}
\end{equation}
The Seq2Seq formulation learns a single mapping $\mathbf{S} \mapsto \mathbf{X}$ that captures the joint distribution over full temporal trajectories, enforcing temporal coherence across the entire evolution. This global formulation is particularly advantageous for dissipative systems with invariant spatial structure, where long-range temporal correlations persist and should be explicitly captured.

\textbf{Training Loss.} SHRED is trained with mean squared error, optionally weighted to emphasize physically relevant spatial regions:
\begin{equation}
\mathcal{L}_{\text{SHRED}} = \frac{1}{T} \sum_{t=1}^{T} \sum_{i=1}^{n} w_i \left( \hat{Y}_{t,i} - Y_{t,i} \right)^2,
\label{eq:shred_loss}
\end{equation}
where $w_i$ may be set to emphasize active regions of the dynamics.

\subsection{Multi-Scale Architecture}
\label{sec:multiscale_backbone}
When deployed on a multi-scale architecture such as SENDAI or Cheap2Rich, the model decomposes the reconstruction into additive LF and HF components:
\begin{equation}
\tilde{\mathbf{u}}'(t) = \tilde{\mathbf{u}}_{\text{LF}}(t) + \tilde{\mathbf{u}}_{\text{HF}}(t),
\label{eq:lf_hf_decomp}
\end{equation}
where the LF pathway captures dominant dynamics learned from simulation with latent-space alignment via a GAN~\cite{goodfellow2014generative}, and the HF pathway resolves fine-scale corrections from sensor residuals. Each pathway produces its own latent code $\mathbf{z}_{\text{LF}}(t) \in \mathbb{R}^{d_z}$ and $\mathbf{z}_{\text{HF}}(t) \in \mathbb{R}^{d_z}$, along with corresponding decoded spatial fields that sum to the full reconstruction. In this multi-scale setting, LAPIS-SHRED deploys \emph{separate temporal models} on the LF and HF latent trajectories, as described in Section~\ref{sec:temporal_model}. All weights for the spatial module---including LF-SHRED, GAN, and HF-SHRED components---are held frozen during subsequent LAPIS-SHRED stages.

\begin{remark}[Spatial module Interchangeability]
\label{rem:backbone_modularity}
The choice of spatial reconstruction module is determined by the application domain and the nature of the available prior. Frame-by-frame SHRED is preferred for chaotic systems where localized temporal windows capture the attractor geometry, while Seq2Seq SHRED is preferred for dissipative systems with global temporal coherence. Multi-scale architectures such as Cheap2Rich and SENDAI are appropriate when a significant domain shift exists and spectral decomposition is beneficial. 
\end{remark}



\section{LAPIS-SHRED Architecture}
\label{sec:method}

We introduce LAPIS-SHRED (\textbf{LA}tent \textbf{P}hase \textbf{I}nference from \textbf{S}hort time sequence), an architecture for reconstructing or forecasting complete spatiotemporal dynamics from sparse sensor measurements confined to a short temporal window. The framework operates entirely in a learned latent space, decoupling the spatial reconstruction module from the temporal inference model, and is designed to interface modularly with any pre-trained SHRED that maps sparse sensor histories to full spatiotemporal states.

The critical challenge addressed by LAPIS-SHRED is the \emph{short-window observation constraint}: in many systems governed by coupled, multiscale physical processes, sensor measurements are available only over a brief temporal segment---either at the terminal end or the initial phase---while the complete spatiotemporal trajectory must be recovered. When only a short terminal window is observed, we seek to reconstruct the antecedent dynamics (\emph{backward inference}). When only a short initial window is observed, we seek to forecast the future evolution (\emph{forward inference}). Both settings preclude direct application of standard time-series reconstruction methods that require dense temporal observations. This constraint arises across diverse domains---from geophysical processes such as landslides where only the deposited material is observed~\cite{iverson1997physics, hungr2014varnes}, to model predictive control (MPC) problems where system trajectories must be inferred from terminal state specifications~\cite{mayne2000constrained, rawlings2008model}, to forecasting of complex multiscale systems such as rotating detonation engines (Section~\ref{sec:rde_results} and ~\ref{sec:1drde_results}) where real-time sensor data is available only over a short observation window and future evolution must be predicted~\cite{raman2023nonidealities, sato2021mixing}, to ecological monitoring (Section~\ref{sec:ndsi_results}) when only a brief, quality-controlled seasonal window of remote sensing observations is available, from which earlier or later phenological stages must be inferred~\cite{weiss2020remote}.


\subsection{Problem Formulation}
\label{sec:problem}

Let $\mathbf{Y}(t) \in \mathbb{R}^{n}$ denote the full spatial state at time $t$, where $n$ denotes the spatial dimensionality of the discretized field. A physics-based simulator $\mathcal{M}$ parameterized by $\boldsymbol{\theta}$ generates the temporal evolution:
\begin{equation}
\{\mathbf{Y}_0, \mathbf{Y}_1, \ldots, \mathbf{Y}_T\} = \mathcal{M}(\boldsymbol{\theta}).
\label{eq:simulation}
\end{equation}
Sparse sensor observations are obtained via a sampling operator $\mathbf{M} \in \mathbb{R}^{p \times n}$ with $p \ll n$:
\begin{equation}
\mathbf{s}_t = \mathbf{M} \, \mathbf{Y}_t \in \mathbb{R}^p.
\label{eq:sensor_sampling}
\end{equation}
Depending on the application, the observation constraint takes one of two complementary forms:

\textbf{Terminal-window observation (backward inference).} A short window of sensor observations $\{\mathbf{s}_{T-W}^{\text{obs}}, \ldots, \mathbf{s}_{T}^{\text{obs}}\}$ is available at the terminal end of the trajectory, where $W \ll T$. In the extreme case, only a single terminal frame $\mathbf{s}_T^{\text{obs}}$ is observed. The reconstruction task seeks a mapping $\mathcal{F}_{\text{bwd}}: \{\mathbf{s}_{T-W}^{\text{obs}}, \ldots, \mathbf{s}_T^{\text{obs}}\} \mapsto \{\hat{\mathbf{Y}}_0, \ldots, \hat{\mathbf{Y}}_T\}$ that recovers the complete antecedent temporal evolution from this short terminal window.

\textbf{Initial-window observation (forward inference).} A short window of observations $\{\mathbf{s}_{t_0}^{\text{obs}}, \ldots, \mathbf{s}_{t_0+W}^{\text{obs}}\}$ is available at the beginning of the trajectory, where $W \ll T$. The forecasting task seeks a mapping $\mathcal{F}_{\text{fwd}}: \{\mathbf{s}_{t_0}^{\text{obs}}, \ldots, \mathbf{s}_{t_0+W}^{\text{obs}}\} \mapsto \{\hat{\mathbf{Y}}_{t_0+W+1}, \ldots, \hat{\mathbf{Y}}_{t_0+W+T_{\text{fwd}}}\}$ that predicts future states beyond the observation window.

\begin{remark}[Extreme Data Constraint]
\label{rem:data_efficiency}
In both settings during inference, \textbf{only the sparse sensor measurements} $\mathbf{s}_t$ from the short observed window are provided to the model---not the full spatial fields at the observed frames, and not any intermediate or future states. These unobserved frames are used exclusively for post-hoc evaluation of reconstruction or inference quality. This represents a severely data-constrained regime: inferring $T$ full spatiotemporal frames from $W+1$ sparse sensor observations with $p \ll n$.
\end{remark}

\subsection{Encoding from the Observed Window}
\label{sec:encoding}

Given sparse sensor observations over a short temporal window, the LAPIS-SHRED encoding strategy depends on the nature and length of the observed sequence.

\textbf{Multi-Frame Window Encoding.} When a short window of $W+1$ sensor frames $\{\mathbf{s}_{t_0}, \ldots, \mathbf{s}_{t_0+W}\}$ is available (with $W > 1$), the frozen temporal unit processes these frames via its standard mechanism---either a sliding-window LSTM (in frame-by-frame mode) or a short-sequence input (in Seq2Seq mode)---to produce latent encodings $\{\mathbf{z}_{t_0}, \ldots, \mathbf{z}_{t_0+W}\}$. No special encoding strategy is required; the pre-trained model simply operates on the available sensor observations.

\textbf{Static Padding for Single-Frame Observation.} In the extreme case where only a single terminal frame $\mathbf{s}_T$ is observed---as may arise in strongly dissipative systems where the asymptotic state encodes sufficient information about the preceding dynamics---the Seq2Seq temporal unit still requires a temporal sequence as input. We resolve this through a \emph{static terminal padding} strategy that exploits the physical stationarity of the terminal state.

Upon cessation of dynamics, the system reaches a stable configuration: $\mathbf{Y}_{T+\delta} \approx \mathbf{Y}_T$ for $\delta > 0$. We construct a pseudo-sequence by replicating the terminal sensor vector:
\begin{equation}
\tilde{\mathbf{S}}_{\text{pad}} = [\underbrace{\mathbf{s}_T^{\text{obs}}, \mathbf{s}_T^{\text{obs}}, \ldots, \mathbf{s}_T^{\text{obs}}}_{L \text{ copies}}]^\top \in \mathbb{R}^{L \times p}.
\label{eq:static_padding}
\end{equation}
This padded sequence is processed by the frozen temporal unit to obtain the terminal latent representation:
\begin{equation}
\mathbf{z}_T^{\text{obs}} = \mathcal{E}(\tilde{\mathbf{S}}_{\text{pad}};\, \phi_{\text{temp}})_{[-1]},
\label{eq:terminal_encoding}
\end{equation}
where $[\cdot]_{[-1]}$ denotes extraction of the final temporal element.

To ensure the model learns to interpret static sequences appropriately, we augment training data with static padding: for each simulation trajectory $\{\mathbf{s}_0, \ldots, \mathbf{s}_T\}$, we append $L$ copies of $\mathbf{s}_T$:
\begin{equation}
\{\mathbf{s}_0, \ldots, \mathbf{s}_T, \underbrace{\mathbf{s}_T, \ldots, \mathbf{s}_T}_{L \text{ copies}}\}.
\label{eq:training_padding}
\end{equation}
The inherent flexibility of the bidirectional LSTM architecture accommodates variable-length inputs, so that the model can process the short $L$-frame padded sequence at inference despite having been trained on full-length sequences of length $T+L+1$. The critical insight is that static padding conditions the temporal unit to recognize constant signals as indicative of stationarity: the terminal latent representation produced from $L$ identical frames approximates the latent state the temporal unit would yield at position $T$ had it observed the complete antecedent trajectory. We employ this strategy in the NDSI backward inference experiment (Section~\ref{sec:ndsi_results}), where a single end-of-season snow-cover observation seeds the reconstruction of the entire preceding snow season.

\subsection{Temporal Dynamics Model}
\label{sec:temporal_model}

With the terminal latent states encoded from the observed short window, we require a mechanism to propagate the dynamics forward or backward in time within the latent space. LAPIS-SHRED supports two complementary temporal model architectures, each suited to different physical regimes and observation constraints.

\subsubsection{Seq2Seq Model (for Fixed-Length Inference)}
\label{sec:backward}
For inference where the output length is known at training time---such as backward reconstruction from a terminal window or forward prediction from an initial window---we train a \emph{sequence-to-sequence dynamics model} that learns the temporal mapping in latent space, producing the complete unobserved latent trajectory in a single forward pass.

\textbf{Architecture.} The Seq2Seq temporal model $\mathcal{B}$ takes the observed latent window and outputs the unobserved latent trajectory. In the backward (reconstruction) setting:
\begin{equation}
\{\hat{\mathbf{z}}_0, \hat{\mathbf{z}}_1, \ldots, \hat{\mathbf{z}}_{T-W-1}\} = \mathcal{B}(\{\mathbf{z}_{T-W}, \ldots, \mathbf{z}_T\}, T;\, \psi),
\label{eq:backward_model}
\end{equation}
where $T$ specifies the output sequence length and the observed latents are appended to form the full trajectory $\hat{\mathbf{Z}} = [\hat{\mathbf{z}}_0, \ldots, \hat{\mathbf{z}}_{T-W-1}, \mathbf{z}_{T-W}, \ldots, \mathbf{z}_T]$. In the forward (prediction) setting, the roles are reversed: the model observes the initial window $\{\mathbf{z}_0, \ldots, \mathbf{z}_{W-1}\}$ and predicts $\{\hat{\mathbf{z}}_W, \ldots, \hat{\mathbf{z}}_T\}$, with the full trajectory assembled as $\hat{\mathbf{Z}} = [\mathbf{z}_0, \ldots, \mathbf{z}_{W-1}, \hat{\mathbf{z}}_W, \ldots, \hat{\mathbf{z}}_T]$. In both cases, the observed latents are encoded directly from sensor data through the frozen SHRED temporal unit and are not re-predicted by $\mathcal{B}$.

Let $T_{\text{out}}$ denote the number of unobserved frames to predict ($T_{\text{out}} = T - W$ in the backward case). The model proceeds in three stages:

\textit{(i) Window compression and input preparation.}
A bidirectional LSTM compresses the observed $W$-frame latent window into a single compression $\tilde{\mathbf{z}}_T$. We employ a bidirectional architecture so it captures context from the entire observation window with equal fidelity, rather than privileging the final frame as a unidirectional architecture would. It is designed to promote a fixed-length compression of the observation window that is independent of the observable window length $W$, capturing salient dynamical information (e.g., phase, frequency, and amplitude).

The summary $\tilde{\mathbf{z}}_T$ is then replicated $T_{\text{out}}$ times and augmented with normalized positional encoding:
\begin{equation}
\tilde{\mathbf{z}}_t^{\text{in}} = \text{Proj}\!\left([\tilde{\mathbf{z}}_T;\, t/(T_{\text{out}}-1)]\right), \quad t = 0, \ldots, T_{\text{out}}-1,
\label{eq:backward_input}
\end{equation}
where $[\cdot\,;\,\cdot]$ denotes concatenation and $\text{Proj}$ is a learned linear projection mapping back to $\mathbb{R}^{d_z}$.

\textit{(ii) Sequence generation.} A second bidirectional LSTM with hidden dimension $d_h^{\mathcal{B}}$ processes the positionally-augmented sequence, propagating information across all timesteps to resolve the full temporal structure of the latent trajectory from the compressed dynamical state. A unidirectional LSTM then scans the resulting sequence in the forward temporal direction, progressively accumulating context to produce hidden states $\{\mathbf{h}_0, \ldots, \mathbf{h}_{T_{\text{out}}-1}\}$ that encode the predicted dynamics at each timestep.

\textit{(iii) Latent projection.} Each hidden state $\mathbf{h}_t$ is mapped back to the latent space through a 2-layer MLP with layer normalization and GELU activation:
\begin{equation}
\hat{\mathbf{z}}_t = \text{MLP}_{\text{out}}\!\left(\mathbf{h}_t\right) \in \mathbb{R}^{d_z}, \quad t = 0, \ldots, T_{\text{out}}-1.
\label{eq:backward_output}
\end{equation}

\textbf{Full Sequence Output.} Unlike autoregressive rollout which accumulates errors over $T$ steps, the model produces the \emph{entire sequence simultaneously}, enabling supervision at every timestep and avoiding accumulative prediction errors.

\textbf{Training Objective.} The training objective is a composite loss:
\begin{equation}
\mathcal{L}_{\text{backward}} = \lambda_r \mathcal{L}_{\text{recon}} + \lambda_s \mathcal{L}_{\text{shape}},
\label{eq:backward_loss}
\end{equation}

\textit{Reconstruction loss} ensures trajectory fidelity:
\begin{equation}
\mathcal{L}_{\text{recon}} = \frac{1}{T} \sum_{t=0}^{T-1} \|\hat{\mathbf{z}}_t - \mathbf{z}_t\|_2^2.
\label{eq:recon_loss}
\end{equation}

\textit{Shape loss} enforces structural similarity between predicted and target latent trajectories:
\begin{equation}
\begin{aligned}
\mathcal{L}_{\text{shape}} = &\frac{1}{T-1} \sum_{t=0}^{T-2} \|(\hat{\mathbf{z}}_{t+1} - \hat{\mathbf{z}}_t) - (\mathbf{z}_{t+1} - \mathbf{z}_t)\|_2^2 \\
&+ \frac{1}{2}\|\text{Var}(\hat{\mathbf{Z}}) - \text{Var}(\mathbf{Z})\|_2^2.
\end{aligned}
\label{eq:shape_loss}
\end{equation}
The shape loss comprises two components: (i) a \emph{temporal difference matching} term that aligns the predicted rate of change at each timestep with the ground truth rate of change, and (ii) a \emph{variance matching} term that ensures the overall dynamic range of the predicted trajectory matches that of the target.

\subsubsection{Autoregressive Model (for Open-Ended Inference)}
\label{sec:forward_ar}

For inference where the prediction horizon may be open-ended or unknown at training time---typically forward prediction---we train an \emph{autoregressive} (AR) latent dynamics model that predicts one step ahead given a lookback window of latent states.

\textbf{Architecture.} The AR model $\mathcal{F}_{\text{AR}}$ maps a lookback window of $W$ consecutive latent states to the next latent state:
\begin{equation}
\hat{\mathbf{z}}_{t+1} = \mathcal{F}_{\text{AR}}\left(\mathbf{z}_{t-W+1:t};\, \xi\right) \in \mathbb{R}^{d_z},
\label{eq:ar_forward}
\end{equation}
where $\mathbf{z}_{t-W+1:t} = [\mathbf{z}_{t-W+1}, \ldots, \mathbf{z}_t] \in \mathbb{R}^{W \times d_z}$. The architecture consists of a bidirectional LSTM that processes the lookback window, followed by an MLP that maps the final hidden state to the predicted latent:
\begin{equation}
\begin{aligned}
\mathbf{h}_W &= \text{BiLSTM}\left(\mathbf{z}_{t-W+1:t};\, \xi_{\text{lstm}}\right)_{[-1]} \in \mathbb{R}^{2d_h}, \\
\hat{\mathbf{z}}_{t+1} &= \text{MLP}\left(\mathbf{h}_W;\, \xi_{\text{mlp}}\right) \in \mathbb{R}^{d_z}.
\end{aligned}
\label{eq:ar_bilstm_mlp}
\end{equation}

\textbf{Coupled LF--HF Forward Models.} When the spatial module decomposes the latent space into LF and HF components (as in the multi-scale setting), we train \emph{separate but coupled} forward models. The LF forward model $\mathcal{F}_{\text{AR}}^{\text{LF}}$ predicts autoregressively without conditioning, while the HF forward model $\mathcal{F}_{\text{AR}}^{\text{HF}}$ is \emph{conditioned} on the predicted LF trajectory:
\begin{align}
\hat{\mathbf{z}}_{t+1}^{\text{LF}} &= \mathcal{F}_{\text{AR}}^{\text{LF}}\left(\mathbf{z}_{t-W+1:t}^{\text{LF}};\, \xi^{\text{LF}}\right), \label{eq:ar_lf} \\
\hat{\mathbf{z}}_{t+1}^{\text{HF}} &= \mathcal{F}_{\text{AR}}^{\text{HF}}\left(\mathbf{z}_{t-W+1:t}^{\text{HF}},\, \hat{\mathbf{z}}_{t+1}^{\text{LF}};\, \xi^{\text{HF}}\right). \label{eq:ar_hf}
\end{align}
This conditioning mirrors the physical coupling in multi-scale architectures where HF residuals depend on the current LF reconstruction.

\textbf{Autoregressive Rollout.} At inference, the model generates predictions of arbitrary length by feeding each predicted state back as input:
\begin{equation}
\hat{\mathbf{z}}_{t+1} = \mathcal{F}_{\text{AR}}\left([\hat{\mathbf{z}}_{t-W+1}, \ldots, \hat{\mathbf{z}}_t]\right), \quad t = W, W+1, \ldots
\label{eq:ar_rollout}
\end{equation}
This sliding-window feedback loop enables generation of latent trajectories of arbitrary length, at the cost of potential error accumulation over longer horizons.

\textbf{Training via Sliding Windows.} Training data consists of (input window, target) pairs extracted at every valid temporal position from each latent trajectory:
\begin{equation}
\mathcal{D}_{\text{AR}} = \bigcup_{k=1}^{K} \left\{ \left(\mathbf{z}^{(k)}_{j:j+W},\, \mathbf{z}^{(k)}_{j+W}\right) : j = 0, \ldots, T_k - W - 1 \right\},
\label{eq:ar_training_data}
\end{equation}
trained with the standard one-step MSE loss. Latent space normalization is applied before training to ensure equitable contribution across latent dimensions.

\begin{remark}[Seq2Seq vs.\ Autoregressive Temporal Models]
\label{rem:seq2seq_vs_ar}
The choice between Seq2Seq and autoregressive temporal models reflects a trade-off between output quality and flexibility. The Seq2Seq model produces all output timesteps simultaneously, avoiding error accumulation but requiring a fixed output length at training time. The autoregressive model generates one step at a time with feedback, enabling open-ended prediction horizons but incurring accumulative errors over long rollouts. The pairings employed in our experiments---Seq2Seq for backward, AR for forward---reflect the natural constraints of each setting, though in principle either temporal model can be applied in either direction.
\end{remark}

\subsection{Frozen Decoder and Spatial Reconstruction}
\label{sec:frozen_decoder}

Both temporal models operate \emph{entirely in latent space}---they predict latent trajectories $\{\hat{\mathbf{z}}_t\}$, not spatial states $\{\hat{\mathbf{Y}}_t\}$. The final mapping from latent to spatial domain is performed by the \emph{frozen decoder} $\mathcal{D}$ with its weights held fixed:
\begin{equation}
\hat{\mathbf{Y}}_t = \mathcal{D}(\hat{\mathbf{z}}_t;\, \phi^*_{\text{dec}}), \quad t = 0, \ldots, T.
\label{eq:frozen_decode}
\end{equation}
This design provides \emph{regularization} (the temporal model is incentivized to produce latents within the decoder's representational space) and \emph{modular training} (the temporal model can be replaced without retraining the spatial module). When the architecture decomposes into LF and HF components, the frozen decoders are applied separately:
\begin{equation}
\hat{\mathbf{Y}}_t = \mathcal{D}_{\text{LF}}(\hat{\mathbf{z}}_t^{\text{LF}}) + \mathcal{D}_{\text{HF}}(\hat{\mathbf{z}}_t^{\text{HF}}).
\label{eq:frozen_decode_dual}
\end{equation}

\subsection{Ensemble Training for Parameter Generalization}
\label{sec:ensemble}

A model trained on a single parameter configuration $\boldsymbol{\theta}^*$ may fail to generalize when deployed on systems with different physical parameters. We address this through ensemble training across multiple parameter regimes.

\textbf{Training Ensemble.} Let $\Theta_{\text{train}} = \{\boldsymbol{\theta}^{(1)}, \ldots, \boldsymbol{\theta}^{(K)}\}$ denote a set of $K$ training parameter configurations spanning the plausible parameter space. For each $\boldsymbol{\theta}^{(k)}$, we generate a simulation trajectory via Eq.~\eqref{eq:simulation}, yielding the training ensemble:
\begin{equation}
\mathcal{D}_{\text{train}} = \bigcup_{k=1}^{K} \left\{ \mathbf{Y}^{(k)}_0, \ldots, \mathbf{Y}^{(k)}_{T_k} \right\}.
\label{eq:ensemble_data}
\end{equation}

\textbf{Joint Training.} Both the SHRED model and the temporal dynamics model are trained on the K simulation trajectories; no ground truth observations are used during training:

\begin{equation}
\begin{aligned}
\phi^* &= \argmin_{\phi} \sum_{k=1}^{K} \mathcal{L}_{\text{SHRED}}^{(k)}(\phi), \\
\psi^* &= \argmin_{\psi} \sum_{k=1}^{K} \mathcal{L}_{\text{temporal}}^{(k)}(\psi).
\end{aligned}
\label{eq:joint_training}
\end{equation}
An alternative to parameter-spanning ensembles is \emph{temporal diversity}: extracting multiple sub-sequences from a single long trajectory with varying start and end points (Section~\ref{sec:rde_results}), which provides coverage of different dynamical regimes within a single parameter configuration.

\subsection{Inference Pipeline}
\label{sec:inference}

\textbf{Backward inference (terminal window).} Given the terminal sensor observations $\{\mathbf{s}_{T-W}^{\text{obs}}, \ldots, \mathbf{s}_T^{\text{obs}}\}$, reconstruction proceeds as: (1)~encode the observed window to obtain terminal latents $\{\mathbf{z}_{T-W}, \ldots, \mathbf{z}_T\}$ (using padding if $W = 0$); (2)~generate $\{\hat{\mathbf{z}}_0, \ldots, \hat{\mathbf{z}}_{T-W-1}\} = \mathcal{B}(\{\mathbf{z}_{T-W}, \ldots, \mathbf{z}_T\}, T;\, \psi^*)$; (3)~form the complete latent trajectory by concatenation; (4)~decode via the frozen decoder.

\textbf{Forward inference (initial window).} Given a short observed window $\{\mathbf{s}_{t_0}^{\mathrm{obs}},\ldots,\mathbf{s}_{t_0+W}^{\mathrm{obs}}\}$, prediction proceeds as: (1)~encode observed frames to obtain $\{\mathbf{z}_{t_0}, \ldots, \mathbf{z}_{t_0+W}\}$; (2)~normalize latent codes; (3)~rollout the forward model for $T_{\text{fwd}}$ steps via Eq.~\eqref{eq:ar_rollout}; (4)~denormalize predicted latents and decode via the frozen decoder. For multi-scale architectures, LF rollout is performed first, followed by LF-conditioned HF rollout, and the decoded spatial fields are summed per Eq.~\eqref{eq:frozen_decode_dual}.

At inference, the pipeline receives its sole input from the true system: a short observation window of hyper-sparse sensor measurements. The architecture operates in a purely feedforward manner to produce the complete temporal reconstruction or forecast, requiring no retraining, fine-tuning, or access to ground truth beyond the observed short window.

\begin{table*}[t]
\centering
\caption{Summary of LAPIS-SHRED experiments and quantitative results. ``Temporal'' indicates the latent dynamics model; ``Dir.'' the inference direction; ``Mode'' the SHRED operational mode. RMSE and NRMSE are reported for the LAPIS-SHRED reconstruction over the full trajectory. SSIM is reported where computed (``--'' otherwise).}
\label{tab:summary_metrics}
\footnotesize
\setlength{\tabcolsep}{4pt}
\renewcommand{\arraystretch}{1.15}
\begin{tabular}{lcccccccc}
\toprule
\textbf{Experiment} & \textbf{Mode} & \textbf{Temporal} & \textbf{Dir.} & \textbf{RMSE} & \textbf{SSIM} & $\boldsymbol{\Delta}$ & \textbf{NRMSE} \\
\midrule
2D KS & Frame & Seq2Seq & Bwd & 0.727 & 0.524 & 15.94 & 0.046 \\
\multirow{2}{*}{2D KF} & \multirow{2}{*}{Frame} & \multirow{2}{*}{Seq2Seq} & \multirow{2}{*}{Bwd} & $|\mathbf{u}|$: 0.179 & $|\mathbf{u}|$: 0.887 & $|\mathbf{u}|$: 4.10 & $|\mathbf{u}|$: 0.044 \\
 & & & & $\omega$: 5.616 & $\omega$: 0.786 & $\omega$: 155.62 & $\omega$: 0.036 \\
\multirow{2}{*}{2D KVS} & \multirow{2}{*}{Seq2Seq} & \multirow{2}{*}{Seq2Seq} & Fwd & 9.7e-4 & 0.564 & \multirow{2}{*}{0.03} & 0.037 \\
 & & & Bwd & 8.7e-4 & 0.638 & & 0.033 \\
HF-RDE & Frame & AR & Fwd & 0.108 & -- & 0.943 & 0.114 \\
\multirow{2}{*}{1D RDE} & \multirow{2}{*}{Seq2Seq} & \multirow{2}{*}{Seq2Seq} & \multirow{2}{*}{Bwd} & $P$: 0.178 & \multirow{2}{*}{--} & $P$: 7.05 & $P$: 0.025 \\
 & & & & $T$: 0.266 & & $T$: 8.40 & $T$: 0.032 \\
\multirow{2}{*}{NDSI} & \multirow{2}{*}{Seq2Seq} & \multirow{2}{*}{Seq2Seq} & Fwd & 0.266 & 0.405 & \multirow{2}{*}{1.59} & 0.167 \\
 & & & Bwd & 0.207 & 0.495 & & 0.130 \\
\bottomrule
\end{tabular}
\end{table*}

\section{Results}
\label{sec:results}

We evaluate the LAPIS-SHRED framework across six experiments spanning chaotic PDEs, vortex-dominated incompressible flows, multiscale propulsion physics, volatile combustion dynamics, and satellite-derived environmental fields. Each experiment exercises a distinct combination of LAPIS-SHRED components---SHRED architecture, temporal model, and inference direction---demonstrating the framework's versatility as a broadly applicable latent-space inference architecture. Table~\ref{tab:summary_metrics} summarizes the experimental configurations and quantitative results.

We report the root mean squared error (RMSE), the Structural Similarity Index Measure (SSIM) where computed, the data range $\Delta = \max - \min$ of the ground truth field, and the normalized RMSE (NRMSE $=$ RMSE$/\Delta$) as the primary metrics throughout.

\subsection{2D Kuramoto--Sivashinsky (Backward Inference)}
\label{sec:2dks_results}

\textbf{Dynamics.} The two-dimensional Kuramoto--Sivashinsky (KS) equation is a canonical model for spatiotemporal chaos, originally derived to describe diffusive--thermal instabilities of laminar flame fronts~\cite{kuramoto1978diffusion,michelson1977nonlinear}. The equation has since found application in modeling thin liquid film instabilities on inclined planes~\cite{nepomnyashchii1974stability} and interfacial dynamics in two-phase flows~\cite{kalogirou2015depth}. In its undamped form on a doubly-periodic domain $[0, L_x) \times [0, L_y)$:
\begin{equation}
u_t + \tfrac{1}{2}(u_x^2 + u_y^2) + \nabla^2 u + \nabla^4 u = 0,
\label{eq:2dks}
\end{equation}
the system exhibits rich spatiotemporal chaos with positive Lyapunov exponents, making it a challenging benchmark for temporal inference from short observation windows.

\textbf{LAPIS-SHRED Configuration.} We employ the \emph{frame-by-frame} mode on $K = 8$ simulation trajectories generated from perturbed initial conditions on a $64 \times 64$ spatial grid with $T = 101$ temporal snapshots. The backward Seq2Seq temporal model infers the preceding $91$ latent frames from the last $10$ frames of sensor observations (approximately 10\% of the trajectory), using $p = 3$ sensors.

\textbf{Results.} Figure~\ref{fig:2dks_results} presents the LAPIS-SHRED backward reconstruction. Over the full trajectory, LAPIS-SHRED achieves RMSE of $0.727$ and SSIM of $0.524$, compared to the SHRED baseline (with access to the full sensor time histories) which achieves RMSE of $0.641$ and SSIM of $0.558$. With a data range of $\Delta \approx 15.94$, the NRMSE of $0.046$ indicates that LAPIS-SHRED recovers the chaotic spatiotemporal evolution to within 5\% of the field's dynamic range from only 10\% of the temporal data and 3 spatial sensors.

\begin{figure}[t]
\centering
\includegraphics[width=\columnwidth]{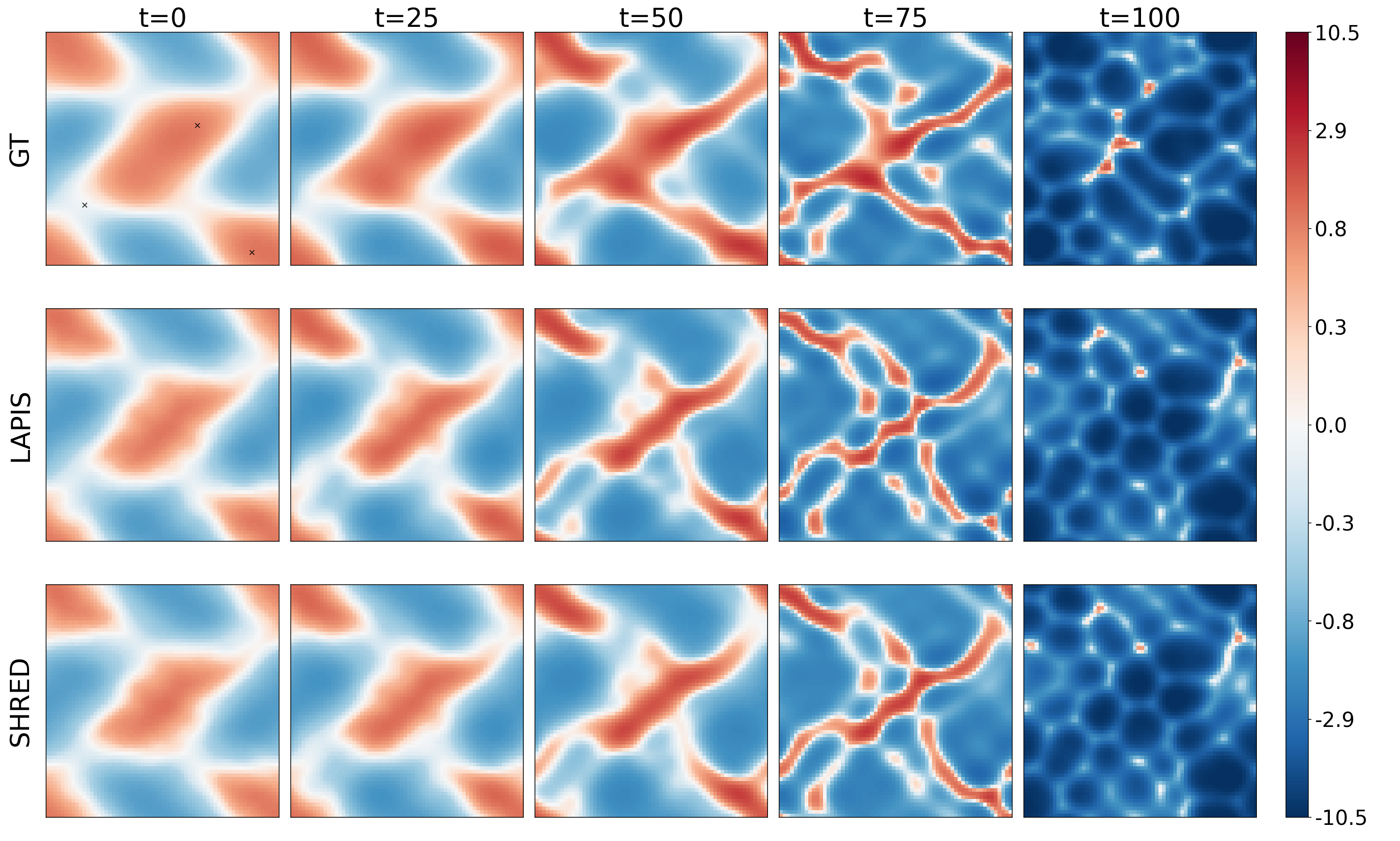}
\caption{LAPIS-SHRED backward reconstruction on the 2D Kuramoto--Sivashinsky system. Ground truth (top row), LAPIS-SHRED reconstruction from the terminal 10\% of the trajectory (second row), and SHRED baseline using the full sensor time-series (third row), shown at selected time steps.}
\label{fig:2dks_results}
\end{figure}

\subsection{2D Kolmogorov Flow (Backward Inference)}
\label{sec:2dkf_results}

\textbf{Dynamics.} The two-dimensional Kolmogorov flow is an incompressible Navier--Stokes system with sinusoidal forcing, a classical benchmark for computational fluid dynamics and data-driven turbulence modeling~\cite{kolmogorov1941dissipation,kochkov2021machine}:
\begin{equation}
\omega_t + \mathbf{u} \cdot \nabla \omega = \frac{1}{Re}\, \nabla^2 \omega + f_\omega, \qquad \mathbf{u} = \nabla^\perp \psi, \quad \nabla^2 \psi = -\omega,
\label{eq:2dkf}
\end{equation}
where $\omega$ is the vorticity, $\psi$ the streamfunction, and $f_\omega = k_0 \cos(k_0 y)$ with $k_0 = 4$ and $Re = 50$. As the Reynolds number increases, the flow undergoes a sequence of bifurcations from laminar to chaotic regimes, exhibiting complex vortex dynamics~\cite{platt1991investigation}

\textbf{LAPIS-SHRED Configuration.} We employ the \emph{frame-by-frame} mode on $K = 15$ simulation trajectories generated from perturbed initial conditions on a $64 \times 64$ spatial grid with $T = 101$ snapshots. The backward Seq2Seq temporal model infers the preceding $91$ frames from the last $10$ frames of sensor observations (approximately 10\% of the trajectory), using $p = 8$ sensors. Both velocity $|\mathbf{u}|$ and vorticity $\omega$ fields are reconstructed.

\textbf{Results.} Figures~\ref{fig:2dkf_velocity} and~\ref{fig:2dkf_vorticity} present the LAPIS-SHRED backward reconstruction for velocity and vorticity fields, respectively. On velocity field over the full trajectory, LAPIS-SHRED achieves RMSE of $0.179$ and SSIM of $0.887$, closely approaching the SHRED baseline RMSE of $0.152$ and SSIM of $0.862$. The consistently high SSIM values indicate that LAPIS-SHRED preserves the spatial structure of the vortical flow fields, capturing coherent vortex patterns even in the temporally unobserved regime.

\begin{figure}[t]
\centering
\includegraphics[width=\columnwidth]{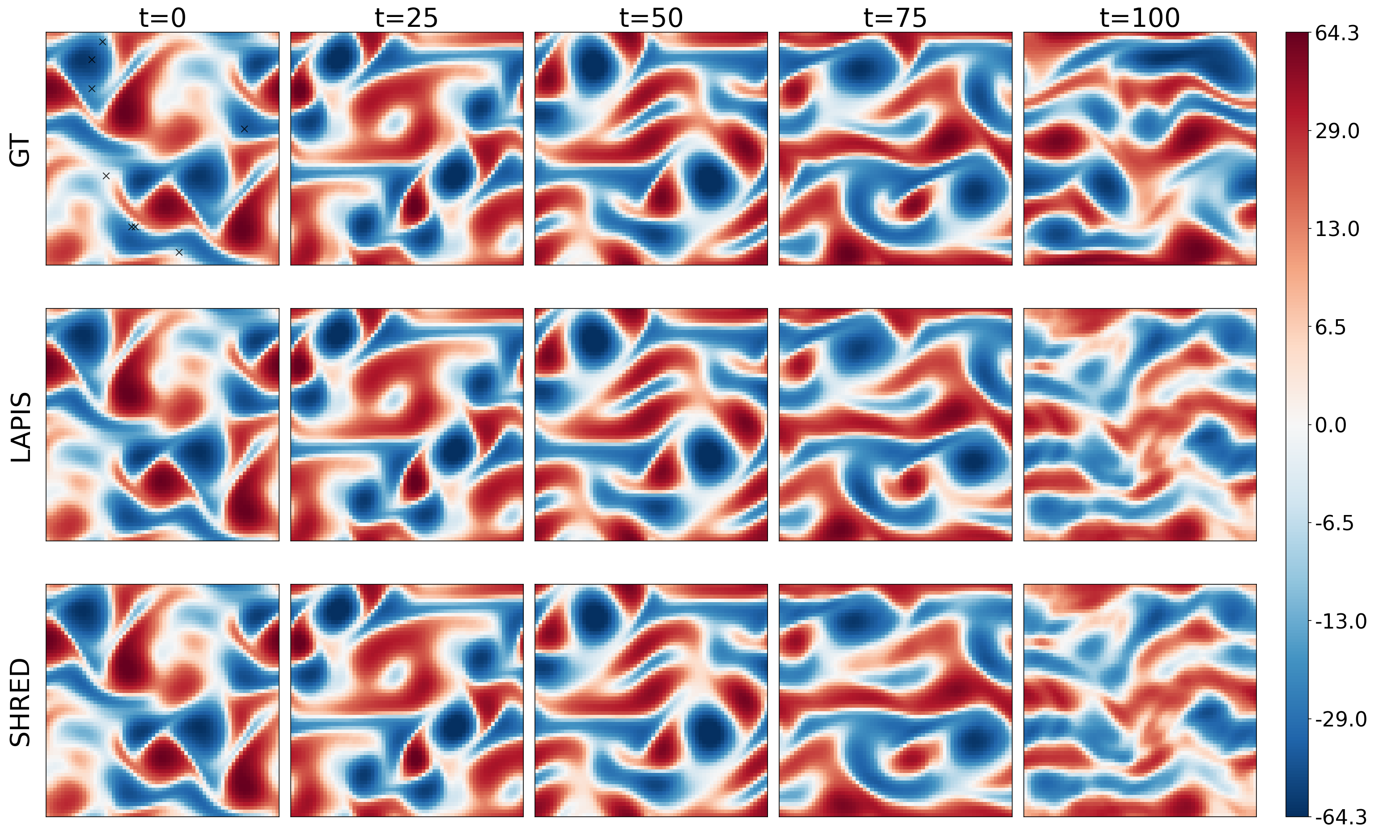}
\caption{LAPIS-SHRED backward reconstruction on 2D Kolmogorov flow: vorticity field. Ground truth (top row), LAPIS-SHRED reconstruction from the terminal 10\% of the trajectory (second row), and SHRED baseline using the full sensor time-series (third row), shown at selected time steps.}
\label{fig:2dkf_vorticity}
\end{figure}

\subsection{2D von Karman Vortex Street (Forward and Backward Inference)}
\label{sec:2dkvs_results}

\textbf{Dynamics.} The von Karman vortex street is a canonical pattern of alternating vortices shed from a bluff body immersed in a viscous flow~\cite{von1911mechanismus,chk1996vortex}. We consider two-dimensional incompressible flow past a circular cylinder, governed by the vorticity transport equation:
\begin{equation}
\omega_t + \mathbf{u} \cdot \nabla \omega = \nu\, \nabla^2 \omega, \qquad \nabla \cdot \mathbf{u} = 0,
\label{eq:2dkvs}
\end{equation}
where $\omega$ is the vorticity, $\mathbf{u}$ the velocity, and $\nu$ the kinematic viscosity. At Reynolds numbers in the range $\mathrm{Re} \approx 60$--$130$, the wake exhibits periodic vortex shedding with Strouhal number $\mathrm{St} \approx 0.165$, producing a low-dimensional limit-cycle attractor.

\textbf{LAPIS-SHRED Configuration.} We employ the \emph{Seq2Seq} mode on $K = 8$ simulation trajectories, with $p = 5$ sensors placed in the wake region. The cylinder interior is excluded from visualization and from all metric computations. We perform two experiments on this dataset, both using $10\%$ of the trajectory as the observation window. Figures~\ref{fig:2dkvs_backward} and~\ref{fig:2dkvs_forward} present the backward and forward reconstruction results, respectively.

\textbf{(i) Forward inference} from a short initial window. The first $14$ frames of the ground-truth sensor sequence are provided, and LAPIS-SHRED predicts the remaining $127$ frames of vortex-shedding evolution. Over the full trajectory, LAPIS-SHRED achieves RMSE of $9.7 \times 10^{-4}$ and SSIM of $0.564$, compared to the SHRED baseline (with access to the full sensor time histories) which achieves RMSE of $3.2 \times 10^{-4}$ and SSIM of $0.882$.

\textbf{(ii) Backward inference} from a terminal observation window. The last $14$ frames are provided, and LAPIS-SHRED reconstructs the preceding $127$ frames. Over the full trajectory, LAPIS-SHRED achieves RMSE of $8.7 \times 10^{-4}$ and SSIM of $0.638$. With a data range of $\Delta \approx 0.03$, the NRMSE values---$0.037$ (forward) and $0.033$ (backward)---confirm that LAPIS-SHRED recovers the vortex-shedding dynamics to within $2.5$--$3.3\%$ of the field's dynamic range from only $5$ sensors and $10\%$ of the temporal data.

\begin{figure}[t]
\centering
\includegraphics[width=\columnwidth]{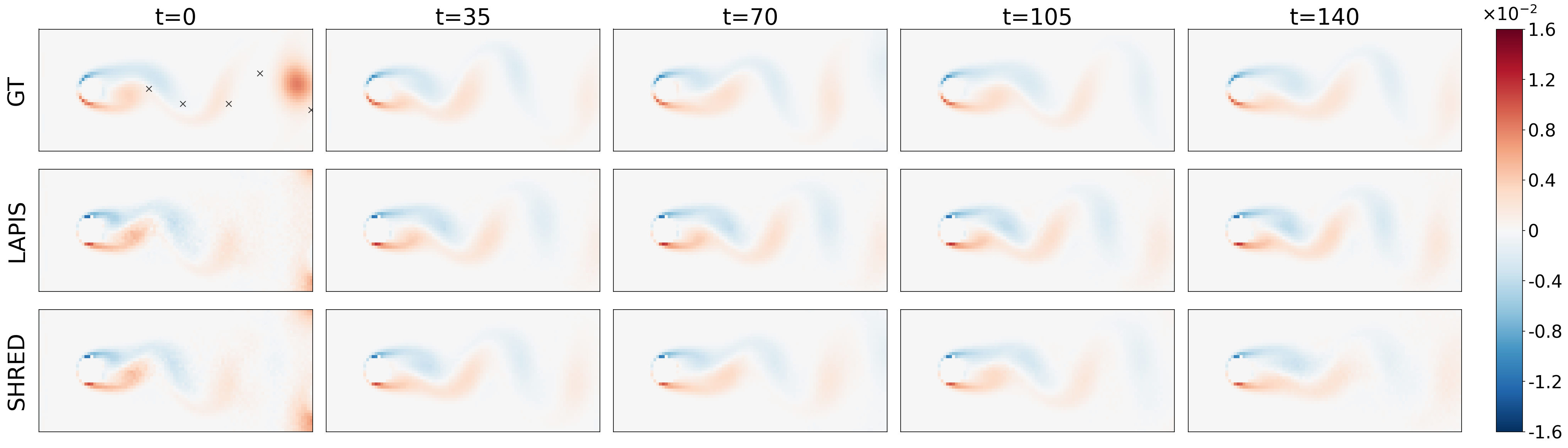}
\caption{LAPIS-SHRED forward prediction on the 2D von Karman vortex street. Ground truth (top row), LAPIS-SHRED prediction from the initial $10\%$ of the trajectory (second row), and SHRED baseline using the full sensor time-series (third row), shown at selected time steps.}
\label{fig:2dkvs_forward}
\end{figure}

\subsection{High-Fidelity RDE (Forward Prediction)}
\label{sec:rde_results}

Rotating detonation engines (RDEs) sustain continuous detonation waves within an annular combustor, offering a pathway toward pressure-gain propulsion with higher thermodynamic efficiency than conventional deflagration-based engines~\cite{raman2023nonidealities}. The underlying physics couples compressible flow, shock--detonation interactions, stiff finite-rate chemical kinetics, and injector-driven mass and energy exchange across disparate spatial and temporal scales. High-fidelity numerical simulation of these coupled processes requires solving the three-dimensional reacting compressible Navier--Stokes equations with detailed chemistry on adaptively refined meshes, at computational costs exceeding $2$~million CPU-hours for a single operating condition~\cite{sharma2024amrex}. This computational burden motivates reduced-order surrogates for design exploration, state estimation, and real-time monitoring.

Two distinct RDE datasets are employed in the LAPIS-SHRED experiments of Section~\ref{sec:rde_results} and ~\ref{sec:1drde_results}. The first is a high-fidelity dataset derived from a three-dimensional simulation of a methane--oxygen rotating detonation rocket engine exhibiting three co-rotating detonation waves, projected onto a one-dimensional azimuthal ring representation ($n = 100$ grid points, $m = 250$ temporal snapshots corresponding to one full wave rotation)~\cite{bao2026cheap2rich}.

\textbf{Dynamics.} We employ the high-fidelity RDE dataset: the preprocessed one-dimensional ring representation ($n = 100$ grid points, $m = 250$ temporal snapshots) released alongside the Cheap2Rich framework. This publicly available dataset was derived from a high-fidelity reacting compressible Navier--Stokes simulation.

\textbf{LAPIS-SHRED Configuration.} We employ the multi-scale architecture operating in frame-by-frame mode, which decomposes reconstruction into a LF pathway (trained on Koch's 1D RDE model~\cite{koch2020modeling}) and a HF pathway capturing fine-scale injector-driven corrections. Coupled autoregressive forward models are trained on the LF and HF latent spaces ($d_z = 32$ each), with the HF model conditioned on predicted LF states (Eqs.~\ref{eq:ar_lf}--\ref{eq:ar_hf}). The lookback window is $W = 25$ latent frames, and 12 overlapping sub-sequences are extracted from the single trajectory to provide temporal diversity. 

\textbf{Validation.} The observation window spans $t = 20.0$ to $t = 22.5$ (frames 200--225), and the forward model predicts 25 frames ahead to $t = 25.0$ against held-out ground truth. Within the observation window, LAPIS-SHRED achieves RMSE of $0.0933$; in the prediction window, RMSE is $0.1212$; the combined full-window RMSE is $0.1081$. This closely matches the Cheap2Rich baseline RMSE of $0.1013$, confirming that the autoregressive rollout maintains reconstruction quality over the short prediction horizon.

\textbf{Extrapolation.} A key capability is forecasting \emph{beyond the training data horizon}. We roll out 200 additional latent steps from $t = 25.0$ to $t = 45.0$---a temporal extent representing 80\% of the original dataset for which \emph{no ground truth exists}. Figure~\ref{fig:rde_extrapolation} presents the results. The LF component (three dominant detonation wavefronts) remains stable and coherent over the full 200-step horizon, while the HF component gradually attenuates---a signature of autoregressive error accumulation that is physically interpretable: the dominant wave dynamics persist while fine-scale stochastic fluctuations regress toward the mean.

The ability to generate 200 frames of physically plausible forward prediction from only 25 frames of observations is significant given that extending the high-fidelity simulation by an equivalent temporal duration would require weeks of supercomputer time and millions of additional CPU-hours. This demonstrates the practical value of LAPIS-SHRED forward inference for computationally expensive multiscale systems: once the model has been trained on available data, LAPIS-SHRED enables low-cost temporal extrapolation that maintains physically consistent dynamics across both dominant wave structures and fine-scale injector-driven features.

\begin{figure}[t]
\centering
\includegraphics[width=\columnwidth]{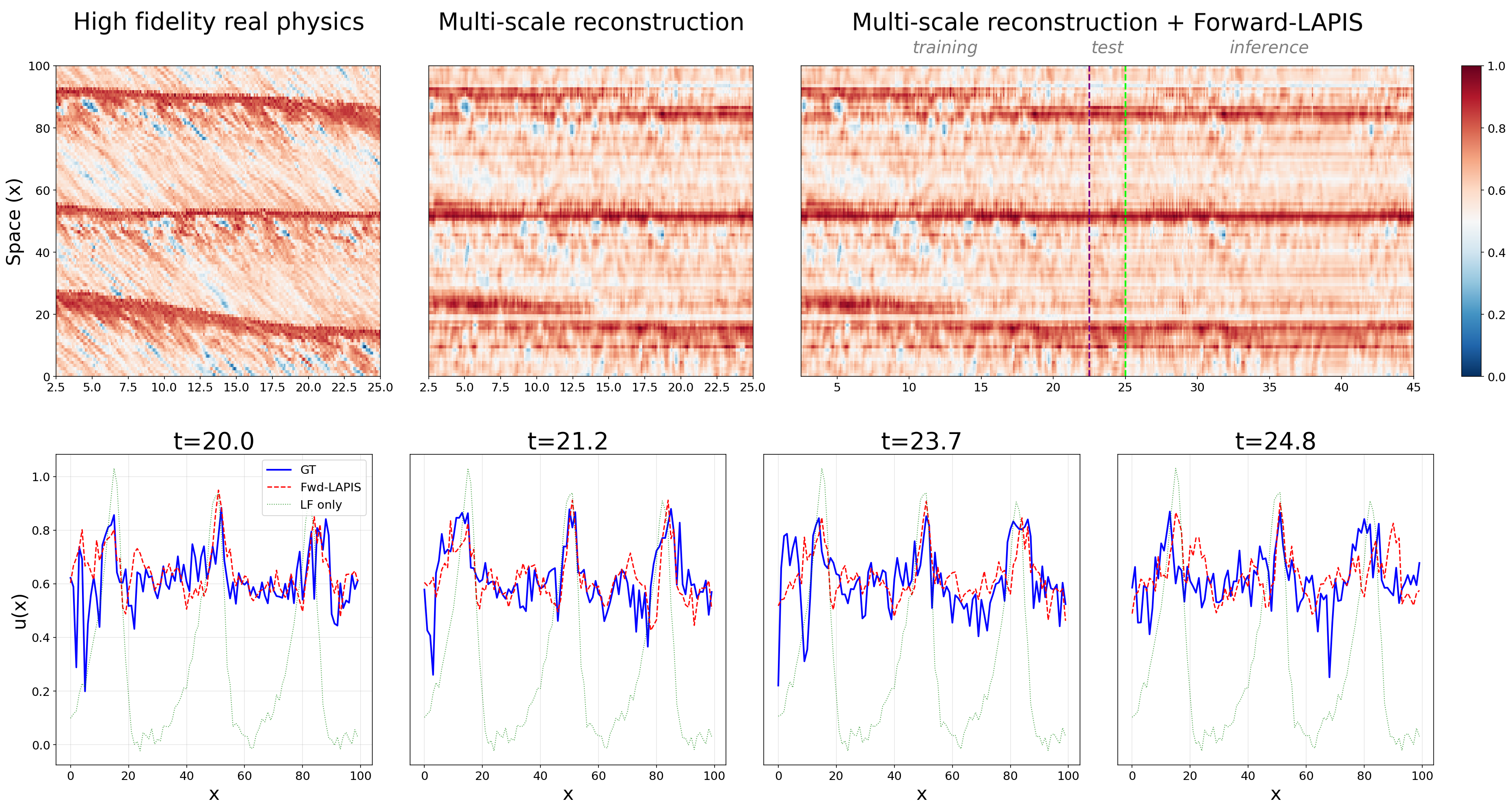}
\caption{Forward-LAPIS-SHRED inference with multi-scale architecture on the 1D rotating detonation engine (RDE) dataset. Top row: spatiotemporal field of the high-fidelity data (left), multi-scale reconstruction (center), and the Forward-LAPIS-SHRED inference (right), which concatenates the training window, a held-out test region, and a 200-step autoregressive forward rollout into the unobserved regime. The purple dashed line marks the training data boundary and the green dashed line marks the observation boundary. Bottom row: spatial snapshots at selected time steps comparing the ground truth, Forward-LAPIS-SHRED inference, and the low-fidelity-only baseline.}
\label{fig:rde_extrapolation}
\end{figure}

\subsection{1D RDE Ignition Stage (Backward Inference)}
\label{sec:1drde_results}

The second RDE dataset is generated by Koch's one-dimensional RDE model~\cite{koch2020modeling}, a reduced-order surrogate based on the reactive Euler equations with source terms representing injection, mixing, and chemical kinetics. Koch's model captures the dominant detonation-front propagation at dramatically reduced cost but omits injector-driven instabilities and fine-scale mixing processes, producing structured discrepancies relative to high-fidelity simulations.

\textbf{Dynamics.} We focus on the \emph{ignition stage}---the transient phase during which detonation waves are initiated and the system transitions from quiescent to active detonation. Unlike the high-fidelity RDE experiment above, which considers the established quasi-steady propagation regime, this regime exhibits volatile thermodynamic behavior including sharp pressure and temperature fronts. The reconstructed fields are the pressure $P(x,t)$ and temperature $T(x,t)$ on a spatial domain with $m_x = 4800$ grid points.

\textbf{LAPIS-SHRED Configuration.} We employ the \emph{Seq2Seq} mode on $K = 5$ simulation trajectories. The temporal domain is clipped to the ignition regime ($t \in [10.0, 20.0]$, yielding $T = 101$ frames) and the spatial domain is downsampled by a factor of 100 to $m_x^{\text{coarse}} = 48$ grid points (state dimension $n = 96$ for the two-field system). Pressure and temperature are treated as separate channels, and $p = 16$ sensors are placed across the spatial domain. The backward Seq2Seq temporal model receives the terminal 20\% of the latent trajectory and reconstructs the preceding 80\%.

\textbf{Results.} Figure~\ref{fig:1drde_results} presents the backward reconstruction. Over the full trajectory, LAPIS-SHRED achieves RMSE of $0.178$ for pressure and $0.266$ for temperature, compared to the SHRED baseline (full sensor time-series) RMSE of $0.162$ and $0.283$, respectively.

The NRMSE values---$0.025$ for pressure ($\Delta_P \approx 7.05$) and $0.032$ for temperature ($\Delta_T \approx 8.40$)---confirm that LAPIS-SHRED recovers the volatile ignition-stage dynamics to within 3\% of each field's dynamic range from only the terminal 20\% of sensor observations.

\begin{figure}[t]
\centering
\includegraphics[width=\columnwidth]{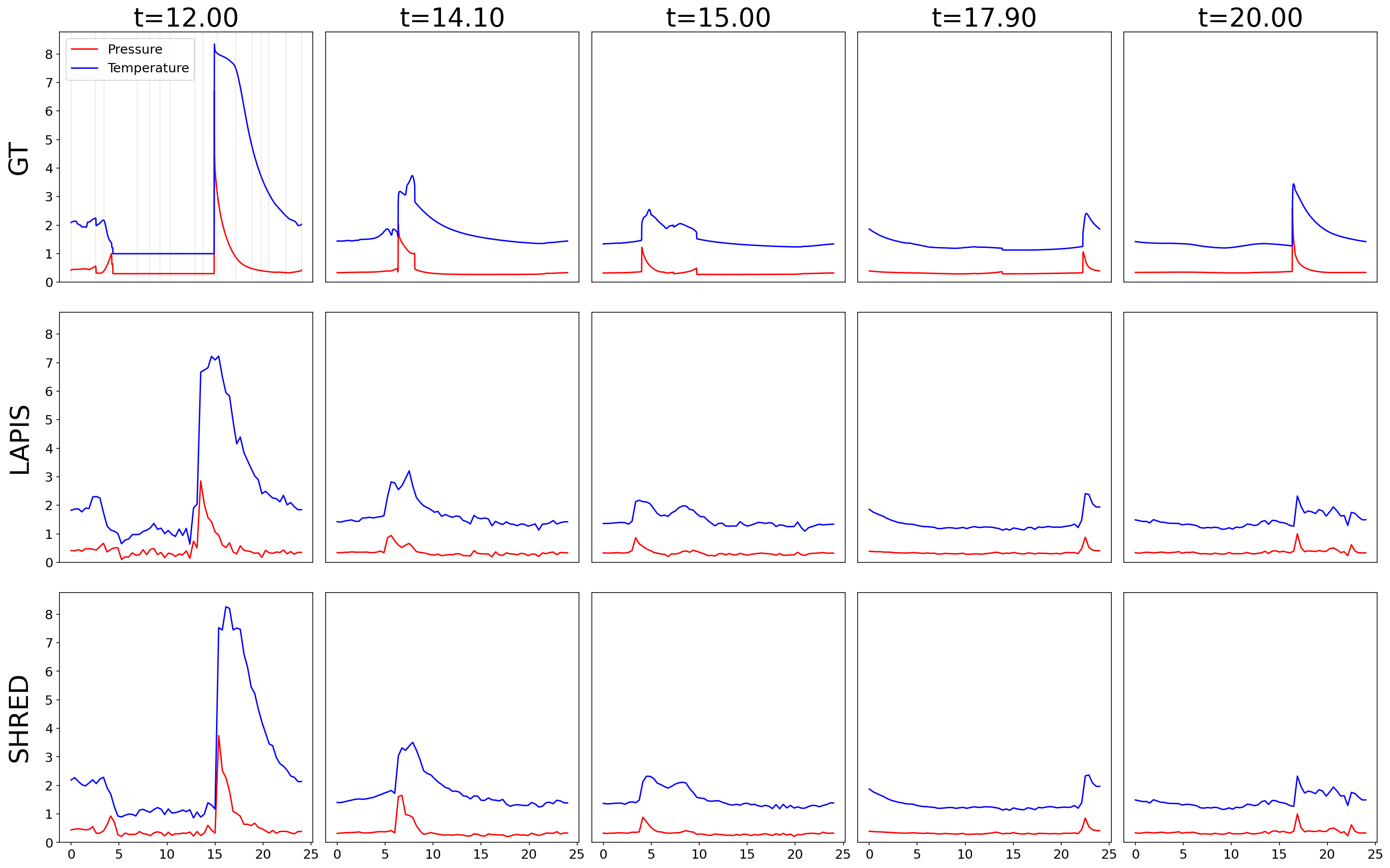}
\caption{LAPIS-SHRED backward reconstruction on the 1D RDE ignition stage. Ground truth (top row), LAPIS-SHRED reconstruction from the terminal 20\% of the trajectory (second row), and SHRED baseline using the full sensor history (third row), shown at selected time steps. Pressure field is shown in red; temperature in blue.}
\label{fig:1drde_results}
\end{figure}

\subsection{2D NDSI Snow Cover (Forward and Backward Inference)}
\label{sec:ndsi_results}

The Normalized Difference Snow Index (NDSI) is a satellite-derived spectral index that quantifies snow coverage by exploiting the contrasting reflectance of snow in visible and shortwave infrared bands~\cite{hall2002modis}. We employ the MODIS daily snow-cover products MOD10A1.061 (Terra) and MYD10A1.061 (Aqua) at 500~m spatial resolution, which provide global daily NDSI maps encoding snow-cover percentage at each pixel (scaled to $[0, 1]$). Invalid observations---including cloud-obscured pixels, nighttime acquisitions, and detector saturation---are masked using the product's quality flags, and missing values are filled via temporal interpolation.

The study domain is a $15 \times 15$~km square region resampled to a standardized $64 \times 64$ pixel grid at fine resolution, centered on the Sierra Nevada mountain range near Lake Tahoe, California ($38.90^\circ$N, $120.10^\circ$W). This region exhibits pronounced seasonal snow dynamics governed by elevation-dependent temperature gradients, orographic precipitation, and solar insolation, producing characteristic spatial patterns of snowline advance and retreat~\cite{dozier2016estimating}. Data are retrieved via Google Earth Engine for December--May snow seasons across multiple years. Five seasons (2020--2024) serve as simulation priors---each year constituting a distinct ``simulation'' trajectory of the same spatial domain under varying meteorological forcing---while the 2025 season serves as the held-out ground truth test year. This experimental design eliminates spatial domain shift between training and test data: all years share identical terrain, land cover, and sensor geometry, with inter-annual variability in precipitation and temperature patterns providing the sim2real gap that LAPIS-SHRED must bridge.

Each seasonal trajectory is trimmed to its active snow-retreat period using a Snow-Covered Area Fraction (SCAF) criterion~\cite{metsamaki2005feasible}, yielding variable-length trajectories (e.g. $T = 37$ for the 2025 ground truth season) that isolate the dynamically active melting period. Additional data generation details are provided in Appendix~\ref{app:ndsi_data}.

\textbf{Dynamics.} We employ the NDSI snow-cover dataset: MODIS-derived daily snow-cover fields over a $64 \times 64$ grid in the Sierra Nevada, with five training years (2020--2024) and the 2025 season as ground truth.

\textbf{LAPIS-SHRED Configuration.} We employ the \emph{Seq2Seq} mode on $K = 5$ observed seasonal trajectories, each trimmed to its active snow-melting period via the SCAF criterion. We use $p = 64$ stratified sensors and a Seq2Seq temporal model.

We perform two experiments on this dataset:

\textbf{(i) Forward inference} from a short initial window. The first 5 frames of the 2025 ground truth sensor sequence (approximately 7\% of the trajectory) are provided, and LAPIS-SHRED predicts the remaining 32 frames of snow-cover evolution. Over the full trajectory, LAPIS-SHRED achieves RMSE of $0.266$ and SSIM of $0.405$, compared to the SHRED baseline RMSE of $0.123$ and SSIM of $0.624$.

\textbf{(ii) Backward inference} from a single terminal frame. Only the final frame of the 2025 ground truth is provided, and static padding (Section~\ref{sec:encoding}) is used to encode the terminal latent state. LAPIS-SHRED then reconstructs the entire preceding seasonal trajectory. Over the full trajectory, LAPIS-SHRED achieves RMSE of $0.207$ and SSIM of $0.495$, yielding NRMSE of $0.130$. This outperforms the forward inference configuration, consistent with the physical structure of the problem: the terminal frame (late spring, near-complete snowmelt) provides a strong constraint on the preceding melt progression, whereas the initial frame (early winter, full snow cover) constrains the subsequent accumulation less tightly due to the stochastic nature of melting events.

Figures~\ref{fig:ndsi_forward} and~\ref{fig:ndsi_backward} present the forward and backward reconstruction results, respectively. The NDSI data preview is provided in Appendix~\ref{app:ndsi_data}.

\begin{figure}[t]
\centering
\includegraphics[width=\columnwidth]{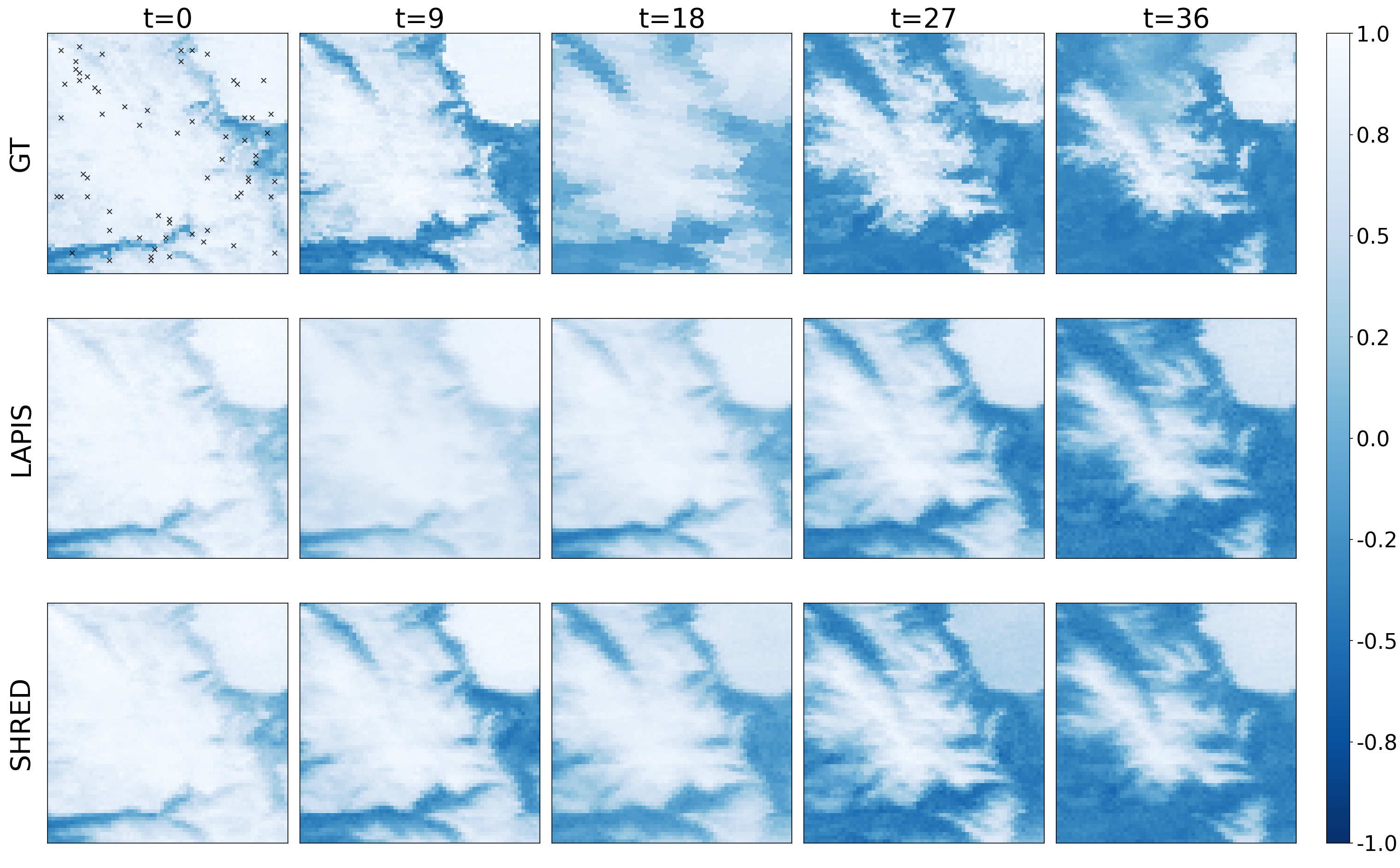}
\caption{LAPIS-SHRED forward inference on the 2D NDSI dataset. Ground truth for the 2025 snow season (top row), LAPIS-SHRED prediction from the initial 5 frames (second row), and SHRED baseline using the full sensor time-series (third row), shown at selected time steps.}
\label{fig:ndsi_forward}
\end{figure}

\begin{figure}[t]
\centering
\includegraphics[width=\columnwidth]{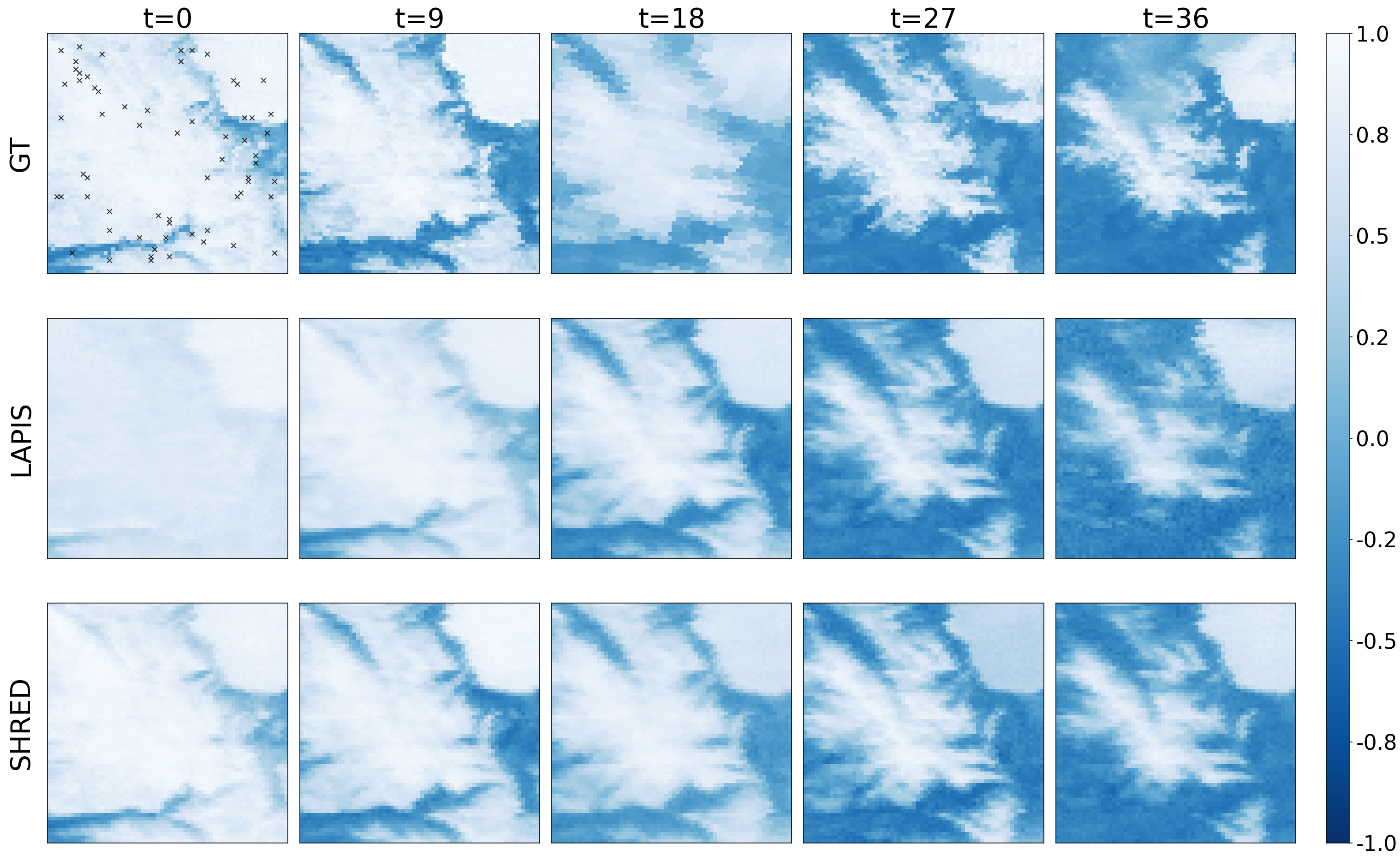}
\caption{LAPIS-SHRED backward inference on the 2D NDSI dataset from a single terminal frame using static padding. Ground truth (top row), LAPIS-SHRED reconstruction (second row), and SHRED baseline using the full sensor time-series (third row), shown at selected time steps.}
\label{fig:ndsi_backward}
\end{figure}

\subsection{Ablation Study on NDSI Data}
\label{sec:ablation}

To understand how each architectural parameter affects LAPIS-SHRED performance, we conduct an ablation study on the NDSI snow-cover
experiment (Section~\ref{sec:ndsi_results}). The axes are chosen to isolate individual components in the error decomposition of Eq.~\ref{eq:error_decomposition}. Figure~\ref{fig:ablation} presents the results.

\begin{figure*}[t]
\centering
\begin{subfigure}[b]{\textwidth}
\centering
    \includegraphics[width=0.8\textwidth]{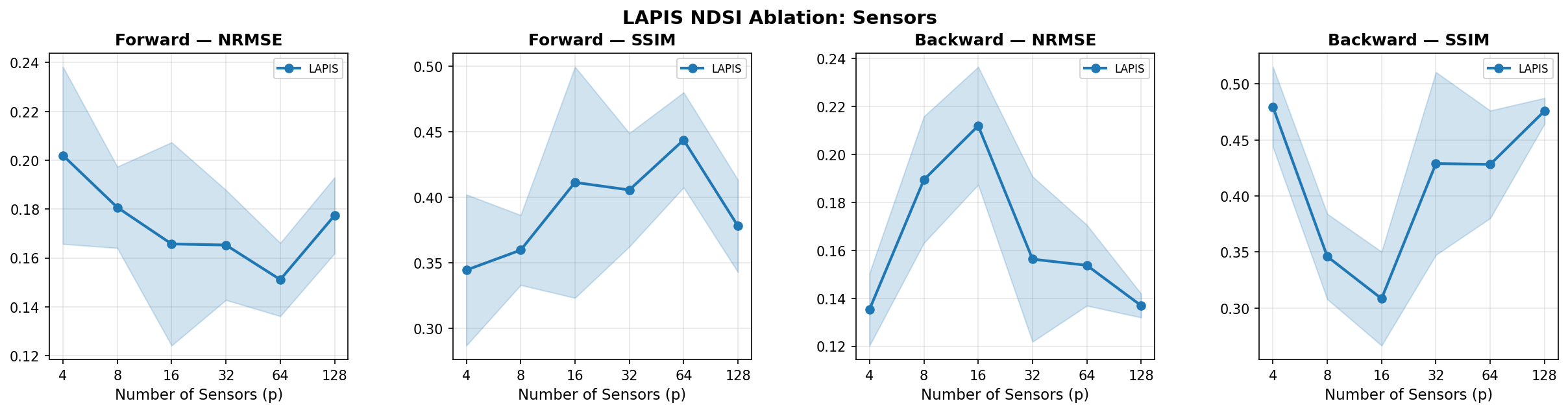}
    \caption{Effect of number of sensors $p$ on LAPIS-SHRED.}
    \label{fig:ablation_sensors}
\end{subfigure}
\begin{subfigure}[b]{\textwidth}
\centering
    \includegraphics[width=0.8\textwidth]{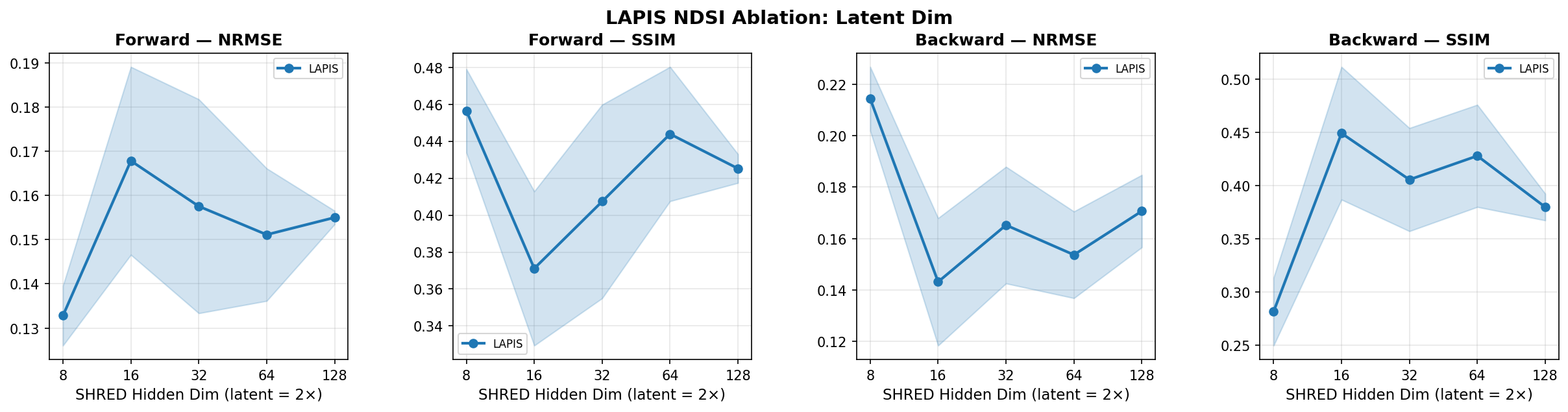}
    \caption{Effect of spatial module latent dimension $d_z$ on LAPIS-SHRED.}
    \label{fig:ablation_latent_dim}
\end{subfigure}
\begin{subfigure}[b]{\textwidth}
\centering
    \includegraphics[width=0.8\textwidth]{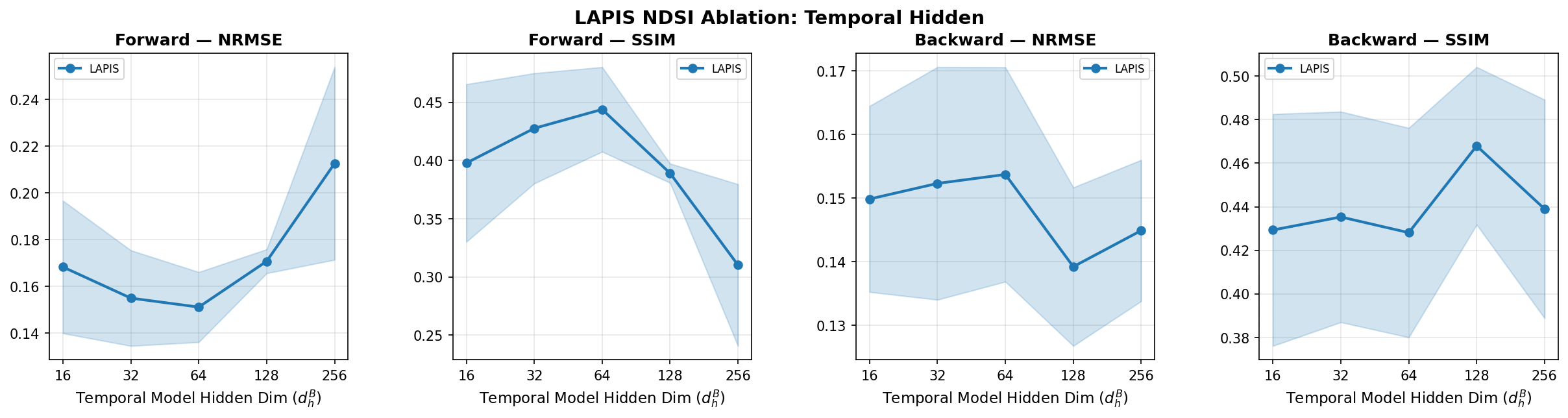}
    \caption{Effect of temporal model hidden dimension $d_h^B$ on LAPIS-SHRED.}
    \label{fig:ablation_temporal_hidden}
\end{subfigure}

\vspace{0.3em}
\begin{subfigure}[b]{0.4\textwidth}
    \includegraphics[width=\textwidth]{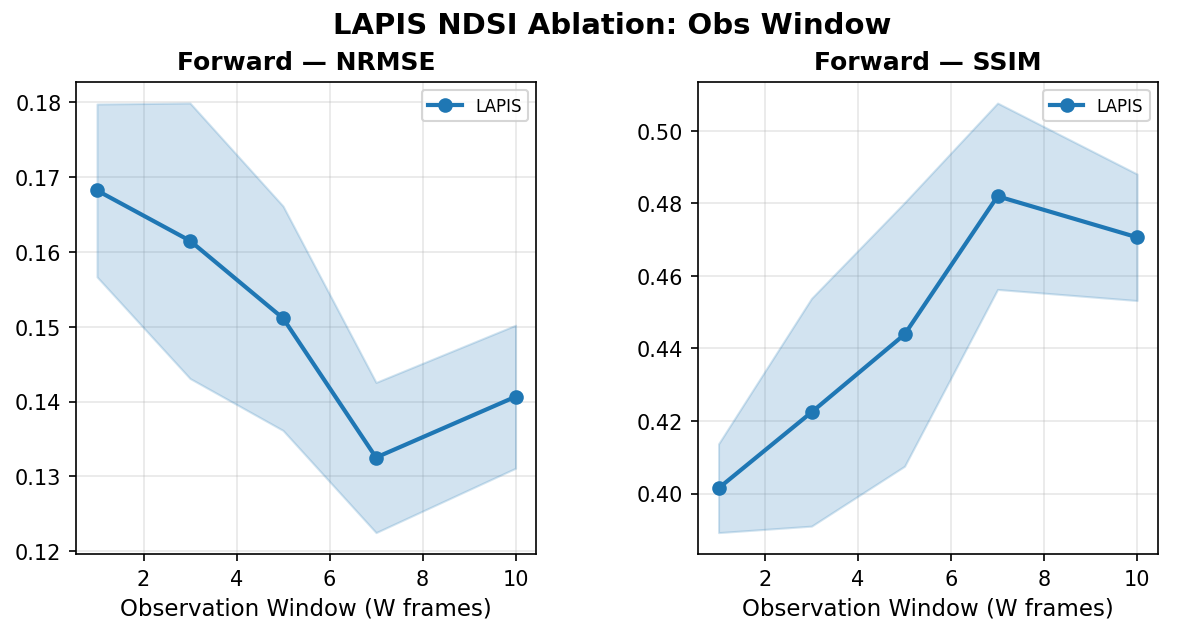}
    \caption{Effect of observation window length $W$\\ on forward inference.}
    \label{fig:ablation_obs_window}
\end{subfigure}
\begin{subfigure}[b]{0.4\textwidth}
    \includegraphics[width=\textwidth]{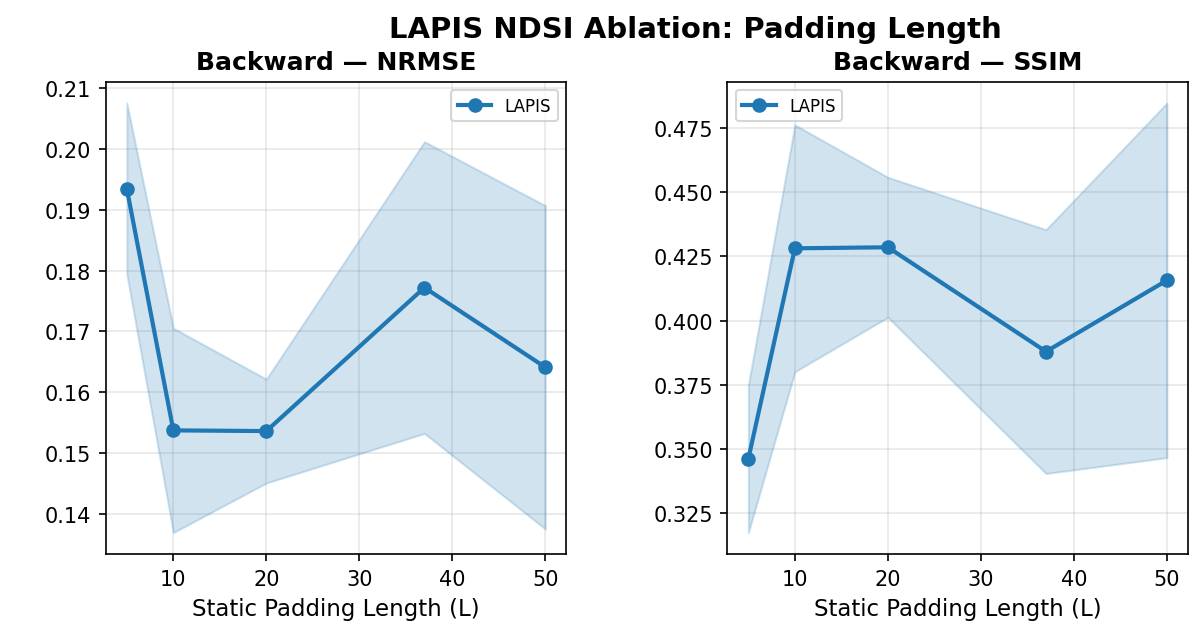}
    \caption{Effect of padding length $L$ on backward inference.\\ \ }
    \label{fig:ablation_padding_length}
\end{subfigure}
\caption{Ablation study on the NDSI dataset. Panels~(a)--(c) report both forward and backward NRMSE and SSIM;
panels~(d) and~(e) are direction-specific. Solid lines denote the mean over three random seeds; shaded bands
indicate $\pm 1$ standard deviation.}
\label{fig:ablation}
\end{figure*}

\textbf{Number of sensors}
Figure~\ref{fig:ablation_sensors} varies the number of sensors $p \in \{4, 8, 16, 32, 64, 128\}$, probing $\varepsilon_{\text{enc}}$ and $\varepsilon_{\text{dec}}$. Both directions show error reduction with increasing $p$, saturating beyond $p \approx 32$. Backward inference exhibits greater performance fluctuation in the low-sensor regime, reflecting the sensitivity of single-frame terminal encoding to spatial sensor configuration.

\textbf{Spatial module latent dimension}
Figure~\ref{fig:ablation_latent_dim} sweeps $d_h^{\text{SHRED}} \in \{8, 16, 32, 64, 128\}$, probing $\varepsilon_{\text{enc}}$ and $\varepsilon_{\text{dec}}$. Forward inference is more robust to small latent dimensions, as the multi-frame observation window provides richer conditioning that partially compensates for reduced per-frame capacity. Performance stabilizes for $d_h^{\text{SHRED}} \geq 32$ in both directions.

\textbf{Temporal model hidden dimension}
Figure~\ref{fig:ablation_temporal_hidden} sweeps $d_h^B \in \{16, 32, 64, 128, 256\}$, probing $\varepsilon_{\text{temp}}(t)$. Forward inference exhibits a clear capacity--generalization trade-off: performance improves up to $d_h^B = 64$ but degrades substantially at $d_h^B = 256$ with increased variance, consistent with overfitting behavior. Backward inference is comparatively insensitive to temporal model capacity, reflecting the stronger physical constraint in dissipative snow-melt dynamics.

\textbf{Observation window length (forward)}
Figure~\ref{fig:ablation_obs_window} varies $W \in \{1, 3, 5, 7, 10\}$ for forward inference, probing $\varepsilon_{\text{enc}}$. Performance improves monotonically from $W = 1$ to $W = 7$, confirming that longer sensor histories yield more faithful embeddings consistent with Takens' theorem.

\textbf{Padding length (backward)}
Figure~\ref{fig:ablation_padding_length} varies $L \in \{5, 10, 20, 37, 50\}$ for backward inference, probing $\varepsilon_{\text{enc}}$ via the terminal encoding strategy of Section~\ref{sec:encoding}. Performance peaks at $L \in [10, 20]$ and degrades for both shorter and longer padding. Too-short padding provides insufficient context for the temporal unit, while too-long padding amplifies the bias from training.

\section{Discussion}
\label{sec:discussion}

In this section, we summarize the experimental results from Section~\ref{sec:results}, connect the empirical observations to the theoretical analysis developed in Appendix~\ref{app:theoretical_foundations}, and discuss the broader implications and generalization potential of the LAPIS-SHRED framework.

\subsection{Summary of Results}
\label{sec:results_summary}

The six experiments demonstrate several key aspects of the LAPIS-SHRED framework:

\textbf{SHRED modularity.} LAPIS-SHRED interfaces seamlessly with frame-by-frame SHRED (2DKS, 2DKF), Seq2Seq SHRED (2D KVS, 1D RDE, NDSI), and multiscale settings (HF-RDE) without modification to the temporal inference stages. This confirms the design principle of Section~\ref{sec:frozen_decoder}: freezing the pre-trained spatial decoder and operating the temporal model entirely in latent space yields a modular architecture in which the spatial reconstruction and temporal inference components can be developed, validated, and replaced independently.

\textbf{Bidirectional temporal inference.} Both backward reconstruction (from terminal windows: 2DKS, 2DKF, 2D KVS-Bwd, 1D RDE, NDSI-Bwd) and forward prediction (from initial windows: 2D KVS-Fwd, HF-RDE, NDSI-Fwd) are demonstrated, with the Seq2Seq temporal model for fixed-length inference and the autoregressive model for open-ended forecasting.

\textbf{Diverse physical regimes.} LAPIS-SHRED operates effectively across spatiotemporally chaotic systems (2DKS, 2DKF), periodic vortex-dominated flows (2D KVS), multiscale propulsion physics with substantial sim-to-real gaps (HF-RDE), volatile combustion transients (1D RDE), and satellite-derived environmental fields with year-to-year variability (NDSI). The successful application to chaotic systems---where positive Lyapunov exponents preclude long-horizon deterministic prediction---demonstrates that LAPIS-SHRED does not require dissipative or equilibrium-convergent dynamics; the simulation prior and short observation window together provide sufficient regularization for latent-space temporal inference even when the underlying dynamics are chaotic.

\textbf{Extreme data efficiency.} Across all experiments, LAPIS-SHRED reconstructs the full spatiotemporal trajectory from only a short observation window---ranging from 7\% (NDSI forward) to 20\% (1D RDE) of the temporal data---using sparse sensor measurements that cover only a small fraction of the spatial domain, while maintaining consistently low NRMSE. Notably, LAPIS-SHRED performance closely approaches that of the corresponding pre-trained SHRED model when it is provided the full temporal sensor history, suggesting that latent-space inference enables accurate trajectory recovery even under severe temporal observation constraints.

\subsection{Baseline Comparisons}
\label{sec:baseline_comparisons}

We contextualize LAPIS-SHRED performance through direct baselines
evaluated under identical conditions and a discussion of broader
methodological classes.

\subsubsection{Performance Relative to SHRED with Full Temporal Observations}
\label{sec:shred_upper_bound}

A natural reference for LAPIS-SHRED is the pre-trained SHRED model
itself when provided the \emph{complete} ground-truth sensor
time-series---an idealized setting in which no temporal sparsity
exists. Figure~\ref{fig:barplot} compares NRMSE and SSIM across all
experiments. Despite operating in an extremely sparse temporal
regime---observing only 7--20\% of the trajectory in general, and in the most extreme case reconstructing an entire
seasonal evolution from a single terminal frame---LAPIS-SHRED
consistently approaches the full-observation upper bound. On the
chaotic PDE benchmarks (2D~KS, 2D~KF) and the combustion physics experiments (HF-RDE, 1D~RDE), LAPIS-SHRED NRMSE remains within 15\% of the SHRED baseline that observes the
complete temporal record. The largest gap appears on the
NDSI dataset, where inter-seasonal variability introduces a
substantial sim-to-real discrepancy; even there, LAPIS-SHRED
recovers the seasonal snow-melt trajectory with NRMSE of $0.130$
(backward from a single frame) and $0.167$ (forward from 5 frames).
These results confirm that latent-space temporal inference recovers
the vast majority of the information contained in a full temporal
sensor record, even under severe observation constraints.

\begin{figure*}[t]
\centering
\includegraphics[width=\textwidth]{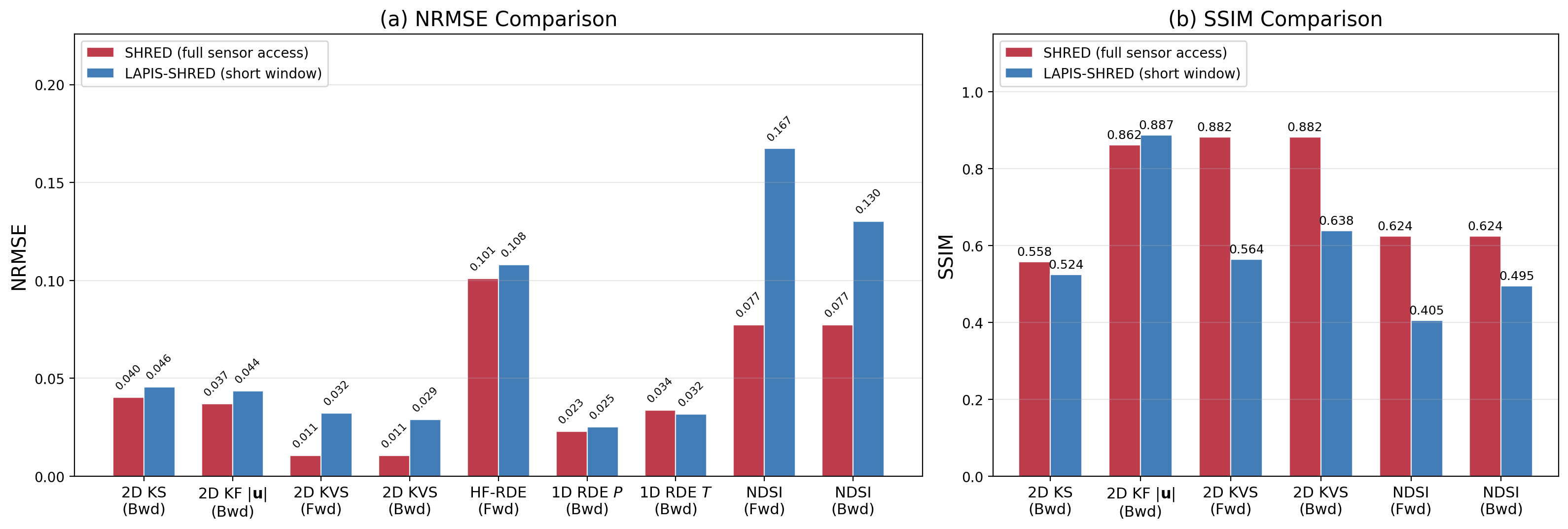}
\caption{LAPIS-SHRED (short observation window) versus SHRED with
full sensor access across all experiments. (a)~NRMSE; lower is
better. (b)~SSIM (where computed); higher is better. LAPIS-SHRED
closely approaches the full-observation baseline despite observing
only 7--20\% of the temporal domain.}
\label{fig:barplot}
\end{figure*}

\subsubsection{SHRED-ROM Baseline}
\label{sec:shredrom_baseline}

SHRED-ROM~\cite{tomasetto2025reduced} combines an LSTM temporal
unit with a shallow decoder mapping to POD coefficients, providing a
compressive reduced-order modeling framework from sparse sensors. As
the most directly comparable existing method, it shares the
foundational principle of sequence modeling of sensor
time-histories for spatial state reconstruction. However, SHRED-ROM does not include a dedicated temporal dynamics model; its reconstruction at each timestep depends on sensor
observations being available at that timestep or in its immediate
lag window. When sensor data is absent beyond a short observation
window, two ad hoc strategies can be employed: (i)~autoregressive sensor-space forecasting, where the model feeds its own predicted sensor values
back as input, and (ii)~ensemble-interpolated sensor construction,
where weights fit from the observed window synthesize sensor
readings from training trajectories. Neither strategy supports
backward reconstruction from terminal observations, as the LSTM processes sensor histories sequentially in the forward temporal direction, and no reverse-time counterpart is provided.

We evaluate both strategies on the HF-RDE and NDSI forward experiments under identical sensor configurations and observation budgets. Table~\ref{tab:baseline_comparison} summarizes the results.
On the HF-RDE dataset, SHRED-ROM achieves comparable aggregate RMSE
but fails to resolve the detonation wavefront structure
(Figure~\ref{fig:shredrom_rde}), producing smoothed fields that lack
the injector physics captured by LAPIS-SHRED. On the NDSI
dataset, both SHRED-ROM strategies yield substantially higher NRMSE
than LAPIS-SHRED. Figures~\ref{fig:shredrom_ndsi} and~\ref{fig:shredrom_ndsi_interp} present the qualitative comparisons.

\begin{table}[t]
\centering
\caption{SHRED-ROM baseline on HF-RDE and NDSI forward
experiments. AR: autoregressive forecast; int.: ensemble-interpolated
sensors. All methods use identical sensors and observation windows.}
\label{tab:baseline_comparison}
\renewcommand{\arraystretch}{1.1}
\small
\begin{tabular}{@{}llccc@{}}
\hline
\textbf{Expt.} & \textbf{Method}
  & \textbf{RMSE} & \textbf{NRMSE} & \textbf{SSIM} \\
\hline
\multirow{2}{*}{\shortstack[l]{HF-\\RDE}}
  & SHRED-ROM (AR)  & 0.110 & 0.117 & --- \\
  & LAPIS-SHRED     & \textbf{0.108} & \textbf{0.114} & --- \\
\hline
\multirow{3}{*}{\shortstack[l]{NDSI\\Fwd}}
  & SHRED-ROM (AR)  & 0.322 & 0.214 & 0.325 \\
  & SHRED-ROM (int.)& 0.304 & 0.191 & 0.391 \\
  & LAPIS-SHRED     & \textbf{0.266} & \textbf{0.167} & \textbf{0.405} \\
\hline
\end{tabular}
\end{table}

\begin{figure}[t]
\centering
\includegraphics[width=\columnwidth]{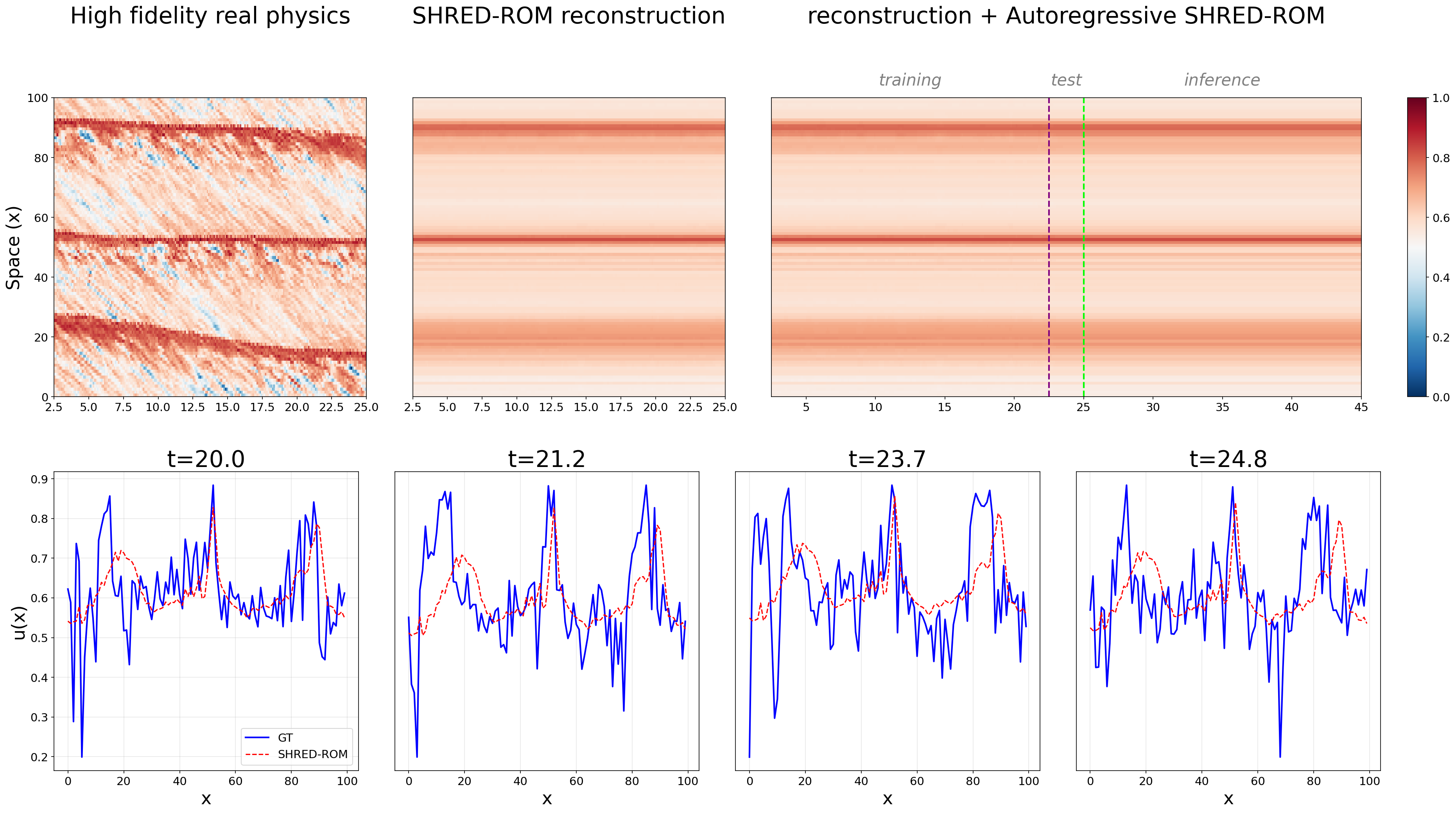}
\caption{SHRED-ROM baseline on the HF-RDE dataset. Top row: ground
truth (left), SHRED-ROM reconstruction (center), and autoregressive
forward rollout (right). Bottom row: spatial snapshots. Compare with
Figure~\ref{fig:rde_extrapolation}.}
\label{fig:shredrom_rde}
\end{figure}

\begin{figure}[t]
\centering
\includegraphics[width=\columnwidth]{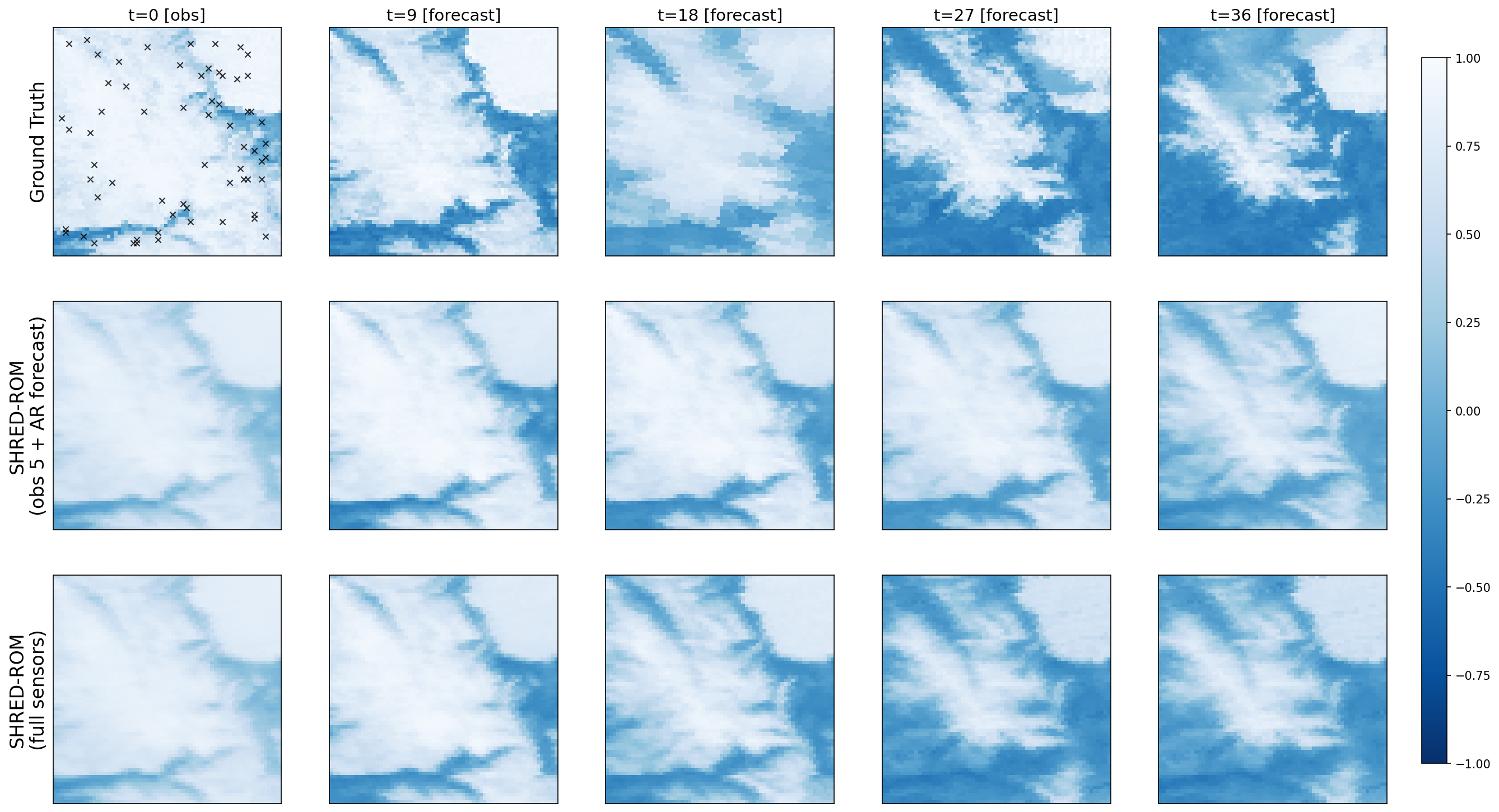}
\caption{SHRED-ROM baseline on the NDSI forward experiment. Ground
truth (top), autoregressive forecast from 5 observed frames (middle),
and SHRED-ROM with full sensor time-series access (bottom). Compare with
Figure~\ref{fig:ndsi_forward}.}
\label{fig:shredrom_ndsi}
\end{figure}

\begin{figure}[t]
\centering
\includegraphics[width=\columnwidth]{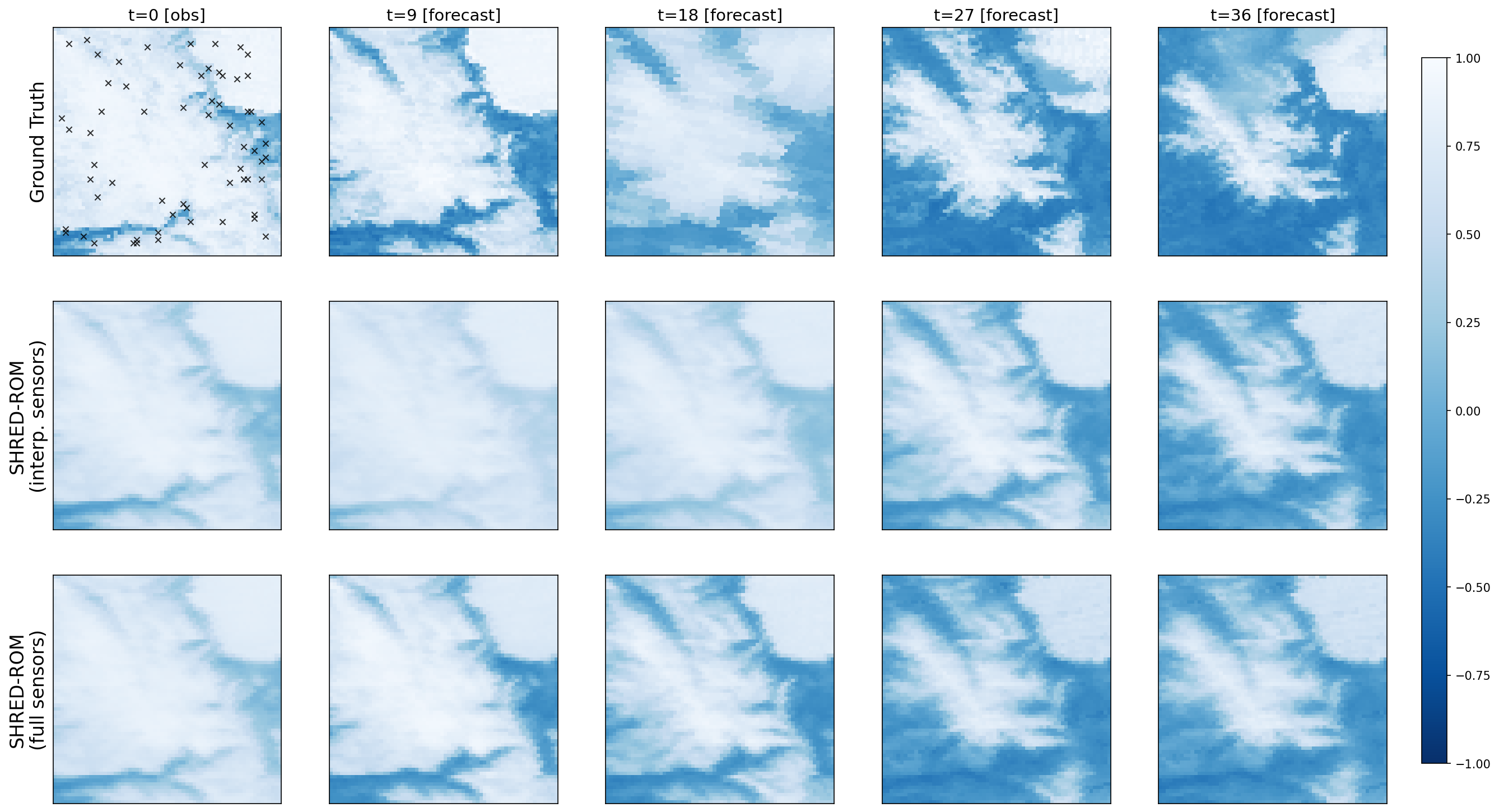}
\caption{SHRED-ROM baseline (ensemble-interpolated) on the NDSI
forward experiment. Ground truth (top), reconstruction from
ensemble-interpolated sensor series (middle), and SHRED-ROM with
full sensor time-series access (bottom). Compare with
Figure~\ref{fig:ndsi_forward}.}
\label{fig:shredrom_ndsi_interp}
\end{figure}

\begin{table}[t]
\centering
\caption{Capability comparison of LAPIS-SHRED and related methods.
\ding{51} = supported; \ding{55} = not supported.}
\label{tab:capability}
\renewcommand{\arraystretch}{1.15}
\small
\begin{tabular}{@{}lcccc@{}}
\hline
\textbf{Method}
  & \rotatebox{70}{\shortstack{Sparse spatial\\sensing}}
  & \rotatebox{70}{\shortstack{Temporal seq.\\encoding}}
  & \rotatebox{70}{\shortstack{Forward\\inference}}
  & \rotatebox{70}{\shortstack{Backward\\inference}} \\
\hline
LAPIS-SHRED        & \ding{51} & \ding{51} & \ding{51} & \ding{51} \\
SHRED-ROM$^\dag$   & \ding{51} & \ding{51} & \ding{51} & \ding{55} \\
PDS                & \ding{51} & \ding{55} & \ding{55} & \ding{55} \\
POD-AE-SE          & \ding{51} & \ding{55} & \ding{55} & \ding{55} \\
POD-DeepONet$^\ddag$ & \ding{51} & \ding{55} & \ding{55} & \ding{55} \\
\hline
\end{tabular}
\vspace{0.3em}
 
\raggedright\footnotesize
$^\dag$Limited forward prediction is possible via autoregressive sensor-space rollout (Section~\ref{sec:shredrom_baseline}), but
SHRED-ROM does not include a dedicated temporal dynamics model.\\
$^\ddag$Standard POD-DeepONet maps explicit parameters to reduced coefficients; adapted to sensor inputs in the SHRED-ROM comparison~\cite{tomasetto2025reduced}.
\end{table}
 
\begin{figure}[t]
\centering
\includegraphics[width=0.85\columnwidth]{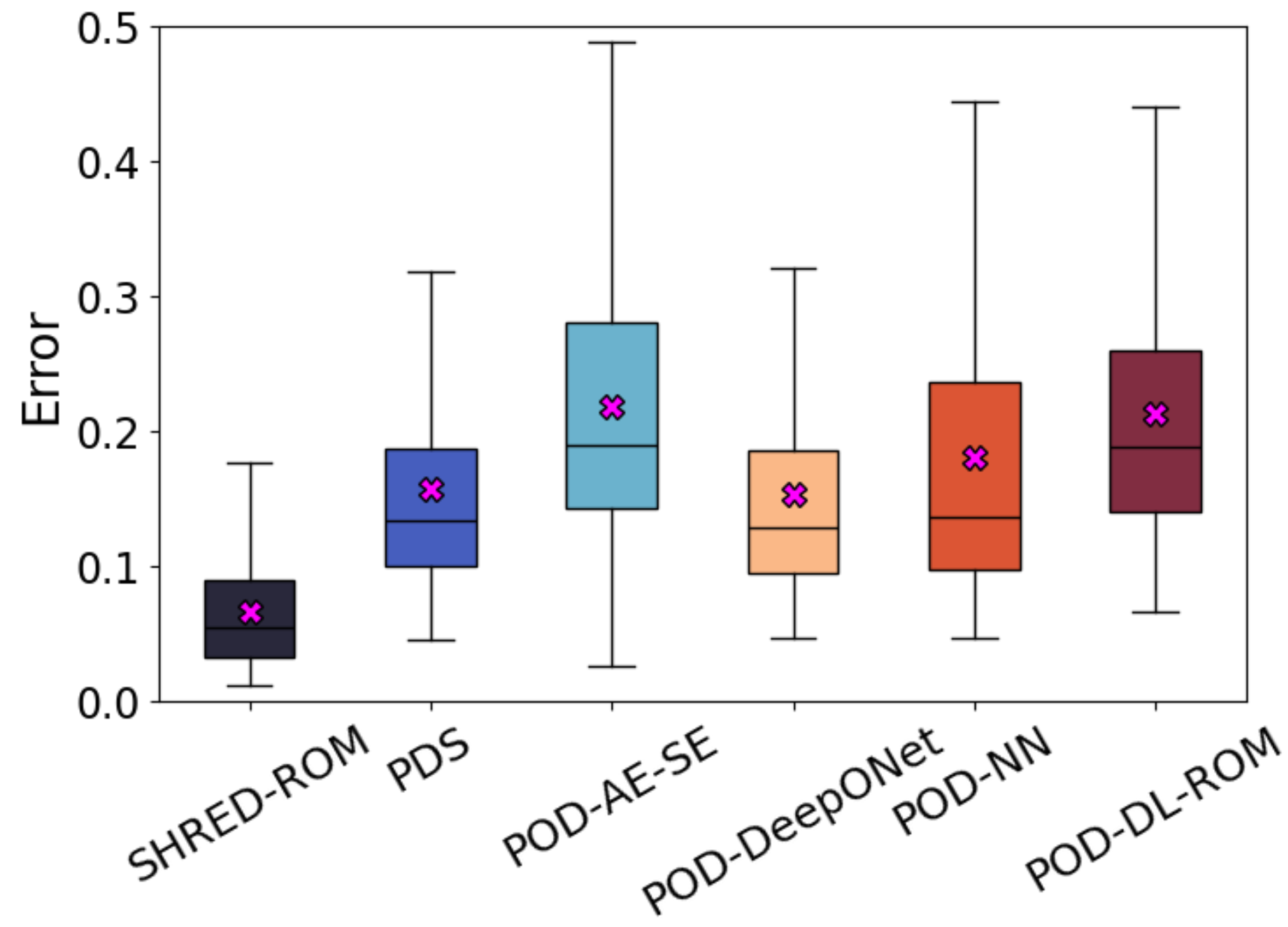}
\caption{Distribution of relative reconstruction errors on the fluid-around-an-obstacle benchmark (reproduced from Tomasetto et
al.~\cite{tomasetto2025reduced}, Figure~7b). SHRED-ROM outperforms all competing methods on the reconstruction task alone; the temporal inference setting addressed by LAPIS-SHRED introduces additional challenges that none of these methods are designed to handle.}
\label{fig:shredrom_comparison}
\end{figure}

\subsubsection{Neural Operator Methods}
\label{sec:neural_operator_discussion}
 
Table~\ref{tab:capability} compares the architectural capabilities of LAPIS-SHRED against related methods, including those benchmarked in the SHRED-ROM study~\cite{tomasetto2025reduced}. Only LAPIS-SHRED supports all four capabilities: spatially sparse sensing, temporal sequence encoding, and both forward and backward temporal inference from short observation windows. SHRED-ROM shares the first two but lacks a dedicated temporal dynamics model for inference beyond the observation window. PDS~\cite{nair2020leveraging} and POD-AE-SE~\cite{luo2023flow} perform one-shot reconstruction from sensor snapshots without temporal encoding, while POD-DeepONet~\cite{lu2022comprehensive} in its standard form maps explicit parameters to reduced coefficients rather than operating from sensor readings.
 
Importantly, even on the reconstruction task alone---where all methods have access to full sensor histories---SHRED-ROM already outperforms PDS, POD-AE-SE, and POD-DeepONet in accuracy, training efficiency, and data requirements, as demonstrated in the comparative evaluation of Tomasetto et
al.~\cite{tomasetto2025reduced} (Figure~\ref{fig:shredrom_comparison}). The temporal inference setting addressed by LAPIS-SHRED is strictly harder: the model must reconstruct or predict the full spatiotemporal trajectory from as few as 1--5 observed frames, a capability that none of these methods are designed to provide.
 
Beyond these specific methods, neural operator frameworks more broadly---including FNO~\cite{li2020fourier} and DeepONet~\cite{lu2021learning}---face three structural mismatches with the LAPIS-SHRED setting. First, standard FNO requires full
spatial fields as input; extensions such as RecFNO~\cite{zhao2024recfno} address spatial sparsity through specialized embeddings but operate on single-timestep snapshots without a temporal propagation mechanism. One could in principle combine such an extension with an autoregressive time-stepper, though to our knowledge this has not been explored for temporal extrapolation from extremely short observation windows; even if realized, the rollout would accumulate errors in the full-field space at every intermediate step, unlike LAPIS-SHRED's latent-space propagation where the Seq2Seq model produces complete trajectories in a single forward pass. Second, extensions such as S-DeepONet~\cite{he2024sequential} that embed recurrent units in the branch network are designed for loading histories that are dense in time and sparse in space---also not the regime considered here. Third, neither FNO nor DeepONet provides a native mechanism for backward temporal inference; FNO's autoregressive rollout processes time sequentially in the forward direction, and DeepONet's trunk-network parameterization does not distinguish temporal direction.
 
In summary, while FNO and DeepONet are powerful frameworks for operator learning, adapting them to the setting addressed by LAPIS-SHRED would require combining sparse spatial sensing ($p \ll n$), support for extreme temporal sparsity (as few as 1--5
observed frames), and bidirectional inference mechanisms---modifications that would substantially alter their architectures, converging toward the recurrent unit and latent-space temporal dynamics design that LAPIS-SHRED provides.

\subsection{Error Analysis}
\label{sec:why_lapis_works}

The effectiveness of LAPIS-SHRED rests on a modular error structure that 
separates encoding fidelity, temporal model accuracy, decoder 
reconstruction, and the simulation--reality gap into independently 
controllable components. In Appendix~\ref{app:theoretical_foundations}, we formalize this decomposition and develop explicit error bounds for 
the analytically tractable case of backward inference in dissipative 
systems converging to stable equilibria 
(Theorem~\ref{thm:nonlinear_dissipative}). While the case study is 
developed in detail for its analytical clarity, the general four-term 
error decomposition (Eq.~\ref{eq:error_decomposition}) applies broadly 
to all LAPIS-SHRED configurations---forward or backward, frame-by-frame or 
Seq2Seq, single-scale or multi-scale---with each configuration instantiating the temporal model error 
$\varepsilon_{\mathrm{temp}}(t)$ differently depending on the dynamical 
regime and observation constraint.

\subsection{Broader Implications and Generalization}
\label{sec:broader_implications}

By supporting both forward and backward inference, accommodating chaotic and dissipative dynamical regimes, and interfacing with diverse SHRED structures, LAPIS-SHRED constitutes a general-purpose latent-space temporal inference architecture applicable wherever dense temporal observations are unavailable but simulation-derived priors exist. We identify in below section several additional candidate application domains beyond our demonstrated experiments.

\subsubsection{Candidate Application Domains}
\label{sec:applications}

\textbf{Solidification and phase transformation.} Casting, welding, and additive manufacturing produce final solidified microstructures observable post-process, while the transient thermal and phase-field evolution during processing is often inaccessible~\cite{karma2001phase,debroy2018additive}. Backward LAPIS-SHRED inference from post-process measurements could reconstruct the thermal history, informing defect analysis and process optimization. When process sensors are available only during an initial heating or melting phase, forward inference could predict the subsequent solidification trajectory.

\textbf{Ecological and environmental monitoring.} Seasonal ecological processes---vegetation phenology, wildfire progression, algal bloom dynamics~\cite{weiss2020remote}---share a similar structure: physics-informed or empirical models can generate plausible trajectory ensembles, while satellite revisit constraints or field restrictions limit observations to narrow temporal windows. Bidirectional inference enables both reconstruction of missed early-season dynamics and forecasting of late-season evolution from partial records.

\textbf{Periodic and rhythmic systems.} Many biological and engineering systems exhibit stable periodic behavior: cardiac dynamics~\cite{glass2001synchronization}, circadian rhythms~\cite{goldbeter1995model}, and periodic chemical reactions~\cite{epstein1998introduction}. When only phase-averaged or stroboscopic observations are available, LAPIS-SHRED reconstruction could recover intra-cycle dynamics from measurements confined to a narrow phase window.

\textbf{Model predictive control with terminal constraints.} In model predictive control (MPC), a common objective is to steer a system toward a reference that may be a static setpoint or a periodic orbit~\cite{rawlings2020model}. When only the terminal specification is given, reconstructing the full optimal state trajectory becomes an inverse problem structurally analogous to backward LAPIS-SHRED inference. Similar architectures could learn mappings from terminal specifications to stabilizing control sequences, bypassing iterative online optimization---a form of amortized MPC~\cite{amos2018differentiable,chen2022large}. Conversely, when initial conditions and a short sensor window are available, forward LAPIS-SHRED inference could serve as a fast surrogate for the predictive model within an MPC loop.

\textbf{Structural health monitoring and forensic engineering.} Progressive failure in structures---fatigue crack growth, corrosion, seismic damage accumulation---produces observable final damage states while intermediate evolution may be unmonitored~\cite{farrar2012structural}. Backward LAPIS-SHRED inference from post-event deformations could reconstruct the transient loading sequences, informing forensic analysis of vehicle collisions or structural collapses.

\section{Summary and Future Directions}
\label{sec:future}

We introduced LAPIS-SHRED, a modular architecture that reconstructs or forecasts complete spatiotemporal dynamics from sparse sensor measurements confined to a short temporal window. Across six experiments spanning chaotic, periodic, transient and multiscale regimes, LAPIS-SHRED recovers full trajectories with consistently low NRMSE from observation windows as short as 7\% or less of the temporal domain and as few as 3 sensors, which closely approaches SHRED baselines which assume sensor observations of the entire time series. The framework supports both backward and forward inference, interfaces with diverse SHRED models, and accommodates extreme observation constraints including single-frame terminal inputs via padding. The modular design provides a general template applicable wherever simulation-derived priors exist but dense temporal observations are unavailable.

Several limitations and future directions remain. First, the framework produces point estimates; extending to calibrated uncertainty quantification---via ensemble predictions, probabilistic latent-space formulations, conformal methods~\cite{psaros2023uncertainty}, or engression~\cite{shen2025engression}---is essential for operational decision support. Second, the observation window is fixed a priori; coupling LAPIS-SHRED with active sensing strategies that optimize observation timing could improve data efficiency~\cite{manohar2018data}. Third, the current theoretical analysis focuses on dissipative systems; extending the error framework to chaotic regimes~\cite{ott2002chaos}, where positive Lyapunov exponents introduce fundamentally different information-theoretic constraints, and to the autoregressive forward setting with cumulative error characterization, would provide a more complete understanding. Fourth, integrating LAPIS-SHRED's temporal inference with multiscale data assimilation frameworks such as SENDAI~\cite{zhang2026sendai} could enable joint state estimation and temporal extrapolation under domain shift, simultaneously addressing the simulation-to-reality gap in spatial reconstruction and the extreme temporal sparsity in observation access. Finally, integrating LAPIS-SHRED with emerging multi-physics foundation models~\cite{mccabe2023multiple} that pretrain across diverse PDE families could replace per-system training with a shared latent-space prior learned across diverse physics, enabling rapid deployment to new dynamical systems.

\section*{Acknowledgments}
The authors were supported in part by the US National Science Foundation (NSF) AI Institute for Dynamical Systems (dynamicsai.org), grant 2112085.  JNK further acknowledges support from the Air Force Office of Scientific Research (FA9550-24-1-0141).

\section*{Code and Data Availability}
Our codebase is available at \url{https://github.com/xswzaqnjimko/LAPIS-SHRED}.

\bibliography{refs}
\bibliographystyle{IEEEtran}

\clearpage

\appendices

\section{Experimental Details: Data Generation and Hyperparameters}
\label{app:experiments}

\subsection{2D Kuramoto--Sivashinsky}
\label{app:2dks}
The 2DKS data is generated using an ETDRK4 pseudo-spectral solver on a
$64 \times 64$ doubly-periodic domain of size $L = 16\pi$,
with $\Delta t = 0.05$, saving every 5 steps ($\Delta t_{\text{save}} = 0.25$)
for a total integration time of $T = 25.0$, yielding 101 snapshots per
trajectory. $K = 8$ simulation trajectories are generated from random
initial conditions with perturbation amplitude $\epsilon = 0.15$.

\subsection{2D Kolmogorov Flow}
\label{app:2dkf}

\begin{figure}[t]
\centering
\includegraphics[width=\columnwidth]{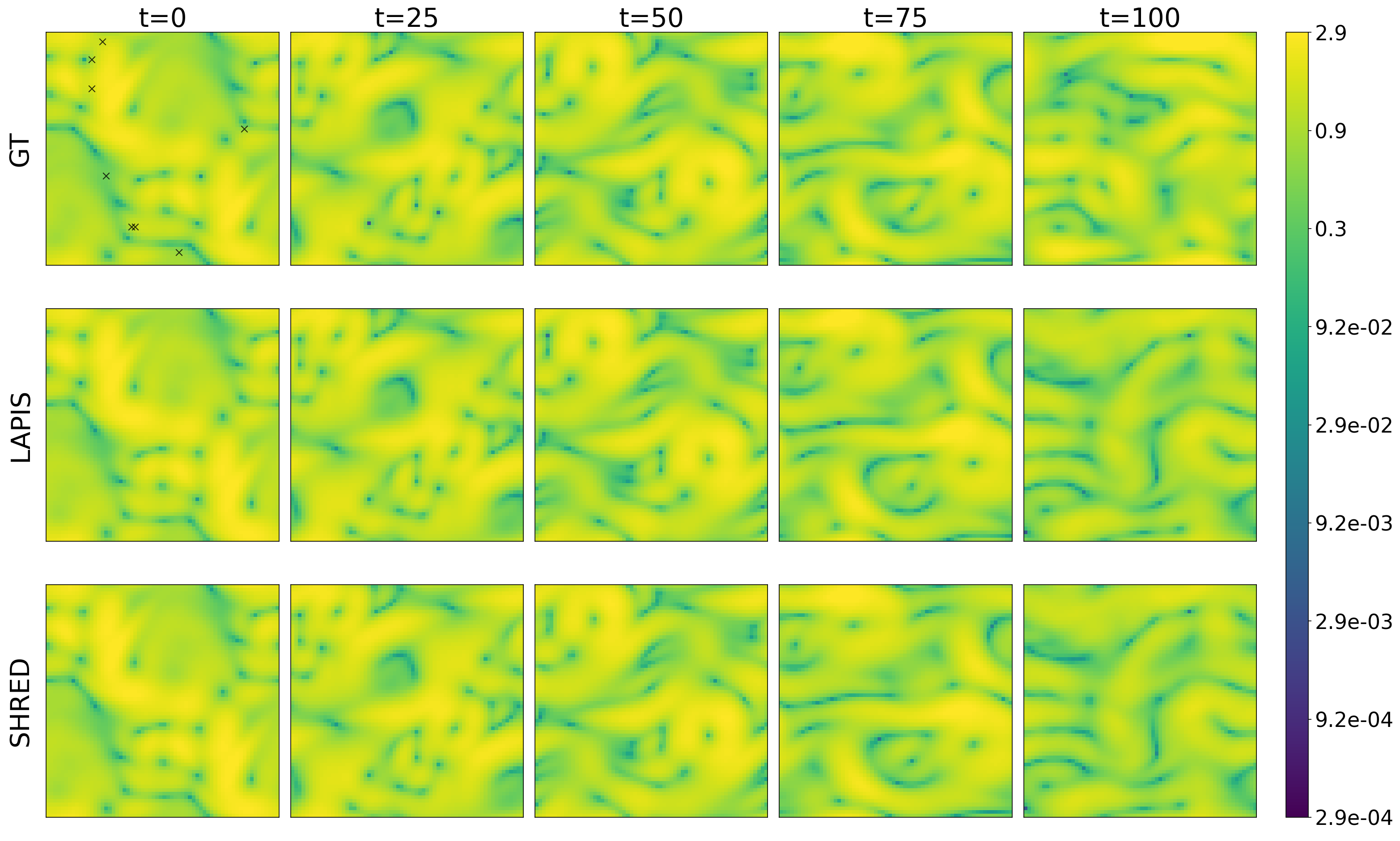}
\caption{LAPIS-SHRED backward reconstruction on 2D Kolmogorov flow: velocity field. Ground truth (top row), LAPIS-SHRED reconstruction from the terminal 10\% of the trajectory (second row), and SHRED baseline using the full sensor time-series (third row), shown at selected time steps.}
\label{fig:2dkf_velocity}
\end{figure}

The 2DKF data is generated using a pseudo-spectral ETDRK4 solver in
vorticity-streamfunction form on a $64 \times 64$ doubly-periodic domain
$[0, 2\pi)^2$ at $Re = 50$, with sinusoidal forcing $f = [0, \sin(4y)] + 0.1$.
After a burn-in period of $T_{\text{burnin}} = 10.0$, snapshots are saved
every 10 steps ($\Delta t_{\text{save}} = 0.1$) for $T_{\text{sim}} = 10.0$,
yielding 101 snapshots. $K = 15$ simulation trajectories are generated with perturbation amplitude
$\epsilon = 0.05$.

\subsection{2D von Karman Vortex Street}
\label{app:2dkvs}

\begin{figure}[t]
\centering
\includegraphics[width=\columnwidth]{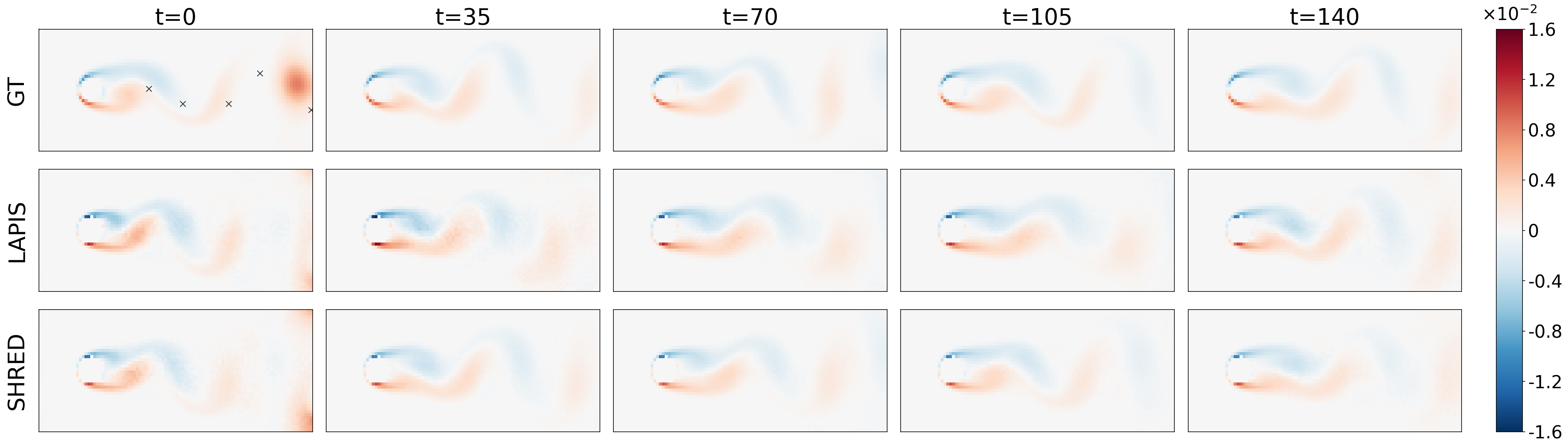}
\caption{LAPIS-SHRED backward reconstruction on the 2D von Karman vortex street. Ground truth (top row), LAPIS-SHRED reconstruction from the terminal $10\%$ of the trajectory (second row), and SHRED baseline using the full sensor time-series (third row), shown at selected time steps.}
\label{fig:2dkvs_backward}
\end{figure}

The 2D KVS data is generated using a D2Q9 lattice Boltzmann method on a
$400 \times 160$ lattice ($r = 16$ cylinder, Zou-He inlet, bounce-back walls).
The periodic vortex street converges to a limit cycle, so ensemble diversity
is achieved via parameter sweep: $K = 8$ simulations with randomized
$\mathrm{Re} \in [60, 130]$, $U_\infty \in [0.03, 0.05]$, inlet modulation,
and burn-in duration ($\tau > 0.535$ enforced). Ground truth uses
$\mathrm{Re} \approx 80.6$, $U_\infty \approx 0.036$. Vorticity snapshots
are saved every 100 steps over 14{,}000 steps ($T = 141$ frames),
downsampled to $100 \times 40$ via $4 \times 4$ block averaging.
$p = 5$ sensors; Seq2Seq SHRED mode ($d_h = 80$);
$\text{obs\_fraction} = 0.10$.

\subsection{High-Fidelity RDE}
\label{app:hf_rde}

The RDE dataset comprises 250 temporal snapshots on a 1D ring of $n = 100$
grid points, with $p = 25$ uniformly distributed sensors and lag $L = 25$.
The multiscale SHRED training follows~\cite{bao2026cheap2rich}: LF-SHRED
for 300 epochs on Koch's 1D model, GAN alignment for 400 epochs, HF-SHRED
with bandlimited sparsity for 500+200 epochs. Forward models use
$d_h = 64$, 2-layer BiLSTM, 3-layer MLP, $W = 25$; training is performed for 600 epochs with AdamW.

\subsection{1D RDE Ignition Stage}
\label{app:1drde}
The 1D RDE data is generated using Koch's model~\cite{koch2020modeling}. $K = 5$ simulation trajectories
span the ignition phase at $t \in [10.0, 20.0]$ ($T = 101$ frames) with spatial
resolution $m_x = 4800$ downsampled to $m_x^{\text{coarse}} = 48$. $p = 16$ sensors; Seq2Seq SHRED mode; backward model
with $\text{obs\_len} = 20$ (20\% tail window).

\subsection{2D NDSI}
\label{app:ndsi_data}
MODIS MOD10A1.061 and MYD10A1.061 daily NDSI products are retrieved via
Google Earth Engine for a $64 \times 64$ grid (500 m) centered on
(38.90$^\circ$N, 120.10$^\circ$W), Sierra Nevada, CA. Five simulation years
(2020--2024) and one ground truth year (2025) are processed. A Snow-Covered Area Fraction (SCAF) with parameters $\tau = 0.4$,
$\rho = 0.25$, $K_{\text{consec}} = 3$ trims each season to its active
snow-cover period. $p = 64$ stratified sensors; Seq2Seq SHRED mode.

\begin{figure}[t]
\centering
\includegraphics[width=\columnwidth]{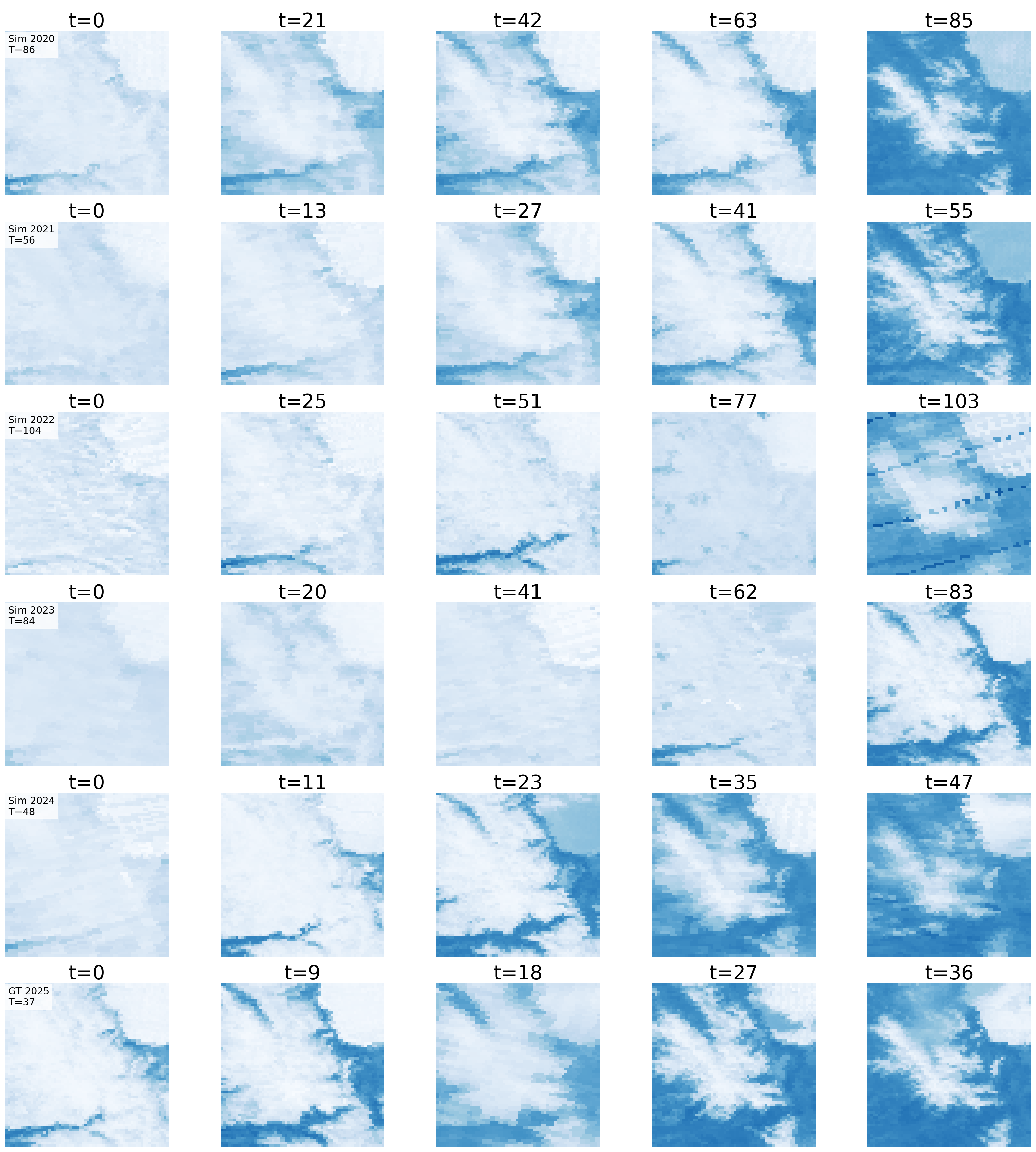}
\caption{NDSI data preview: spatiotemporal structure of the 2020-2025 snow melt seasons after SCAF trimming.}
\label{fig:ndsi_data}
\end{figure}


\section{Theoretical Analysis of LAPIS-SHRED}
\label{app:theoretical_foundations}

This appendix provides a theoretical analysis of the LAPIS-SHRED inference pipeline and develops a detailed case study for one specific and analytically tractable setting: backward inference from terminal observations in dissipative systems converging to stable equilibria. This case study, which encompasses the NDSI backward inference experiment in the main text, is chosen because it 
admits explicit bounds that clarify the roles of the simulation 
ensemble, the encoding strategy, and the temporal model capacity in 
controlling reconstruction accuracy. The general error decomposition (\cref{app:problem_setting}) and the assumptions therein are not restricted to this setting.

\subsection{Problem Setting and Error Decomposition}
\label{app:problem_setting}

Consider an autonomous dynamical system on a Banach space $X$ (e.g., a sufficiently regular Sobolev space $H^s(\Omega)$)~\cite{evans2022partial,robinson2003infinite}:
\begin{equation}
\label{eq:app_dynamics}
\frac{d\mathbf{Y}}{dt} = \mathcal{F}(\mathbf{Y}), \qquad \mathbf{Y}(0) = \mathbf{Y}_0 \in X,
\end{equation}
with flow map $\Phi_t : X \to X$ defined by $\Phi_t(\mathbf{Y}_0) = \mathbf{Y}(t)$. The terminal state is $\mathbf{Y}_T = \Phi_T(\mathbf{Y}_0)$. LAPIS-SHRED seeks to approximate the inverse flow map $\Phi_T^{-1}$ in a latent representation, recovering the trajectory $\{\mathbf{Y}(t)\}_{t \in [0,T]}$ from sparse observations of $\mathbf{Y}_T$.

The total reconstruction error of LAPIS-SHRED at time $t$ decomposes as:
\begin{equation}
\label{eq:error_decomposition}
\underbrace{\|\hat{\mathbf{Y}}(t) - \mathbf{Y}^{\mathrm{obs}}(t)\|}_{\text{total error}} \leq \underbrace{\varepsilon_{\mathrm{enc}}}_{\substack{\text{observation} \\ \text{encoding}}} + \underbrace{\varepsilon_{\mathrm{temp}}(t)}_{\substack{\text{temporal} \\ \text{model}}} + \underbrace{\varepsilon_{\mathrm{dec}}}_{\substack{\text{spatial} \\ \text{decoding}}} + \underbrace{\varepsilon_{\mathrm{sim}}}_{\substack{\text{sim-to-real} \\ \text{gap}}},
\end{equation}
where:
\begin{itemize}[nosep,leftmargin=1.5em]
\item $\varepsilon_{\mathrm{enc}}$ is the error in encoding the observed window (or terminal observation $\mathbf{s}_T^{\mathrm{obs}}$) into the latent space (\cref{sec:encoding});
\item $\varepsilon_{\mathrm{temp}}(t)$ is the error of the temporal dynamics model in recovering or predicting the latent state at time $t$ from the encoded observation window---this encompasses both the backward model $\mathcal{B}$ inferring earlier states and the forward autoregressive model $\mathcal{A}$ predicting later states
\item $\varepsilon_{\mathrm{dec}}$ is the reconstruction error of the frozen decoder mapping latent states to spatial fields;
\item $\varepsilon_{\mathrm{sim}}$ quantifies the distributional discrepancy between simulation training data and ground truth if no data assimilation methods are applied.
\end{itemize}

In the analysis that follows, we focus primarily on $\varepsilon_{\mathrm{enc}}$ and $\varepsilon_{\mathrm{temp}}(t)$, as these are the components most sensitive to the dynamical regime of the system. The terms $\varepsilon_{\mathrm{dec}}$ and $\varepsilon_{\mathrm{sim}}$ depend on the expressivity of the decoder network and the fidelity of the simulation ensemble, respectively, and are treated as controlled quantities.

\begin{assumption}[Well-posedness]
\label{ass:well_posed}
The forward dynamics \eqref{eq:app_dynamics} admit a unique solution $\mathbf{Y}(t) \in C([0,T]; X)$ for all initial conditions in a bounded set $\mathcal{U} \subset X$. The flow map $\Phi_t$ is continuously differentiable with respect to $\mathbf{Y}_0$ for $t \in [0,T]$.
\end{assumption}

\begin{assumption}[Latent representation fidelity]
\label{ass:latent_fidelity}
The SHRED temporal unit $\mathcal{E}_{\mathrm{SHRED}}$ and decoder $\mathcal{D}_{\mathrm{SHRED}}$ define a latent space $\mathcal{Z} \subseteq \mathbb{R}^{d_z}$ such that the composite map $\mathcal{D}_{\mathrm{SHRED}} \circ \mathcal{E}_{\mathrm{SHRED}}$ approximates the identity on the training manifold to accuracy $\varepsilon_{\mathrm{rec}} \ll 1$. That is, for trajectories in the training distribution, $\|\mathcal{D}_{\mathrm{SHRED}}(\mathcal{E}_{\mathrm{SHRED}}(\mathbf{Y}(t))) - \mathbf{Y}(t)\| \leq \varepsilon_{\mathrm{rec}}$ uniformly in $t$.
\end{assumption}

\subsection{Case Study---Dissipative Systems Converging to Stable Equilibria}
\label{app:class_equilibria}

The most natural setting for LAPIS-SHRED is the class of dissipative dynamical systems that converge to isolated stable fixed points. This class encompasses the NDSI snow-cover experiment studied in the main text as well as a broad range of physical processes.

\begin{definition}[Dissipative system with stable equilibrium]
\label{def:dissipative}
The system \eqref{eq:app_dynamics} is \textbf{dissipative with stable equilibrium} $\mathbf{Y}^* \in X$ if:
\begin{enumerate}[label=(\roman*),nosep]
\item There exists a bounded absorbing set $\mathcal{B} \subset X$ such that for every bounded set $U \subset X$, there exists $t_0(U)$ with $\Phi_t(U) \subset \mathcal{B}$ for all $t \geq t_0$.
\item $\mathcal{F}(\mathbf{Y}^*) = \mathbf{0}$, and the Fr\'{e}chet derivative $D\mathcal{F}|_{\mathbf{Y}^*}$ has spectral bound $s(D\mathcal{F}|_{\mathbf{Y}^*}) = \sup\{\mathrm{Re}(\lambda) : \lambda \in \sigma(D\mathcal{F}|_{\mathbf{Y}^*})\} = -\gamma < 0$, where $\gamma$ quantifies the decay rate of perturbations.
\end{enumerate}
\end{definition}

\subsubsection{Linear Systems}

We begin with the simplest case. If the dynamics are linear, $\mathcal{F}(\mathbf{Y}) = \mathcal{L}\mathbf{Y}$ with $\mathcal{L}$ a bounded linear operator satisfying $s(\mathcal{L}) = -\gamma < 0$, the flow map is $\Phi_t = e^{t\mathcal{L}}$ and the terminal state is $\mathbf{Y}_T = e^{T\mathcal{L}}\mathbf{Y}_0$. The inverse flow map is $\Phi_T^{-1} = e^{-T\mathcal{L}}$, which exists as a bounded operator on $X$. However, inverting a dissipative flow is \emph{anti-dissipative}: the operator $e^{-T\mathcal{L}}$ has spectral radius $e^{\gamma T}$, so small perturbations in $\mathbf{Y}_T$ are amplified exponentially in backward time.

\begin{proposition}[Backward sensitivity for linear dissipative systems]
\label{prop:linear_backward}
Let $\mathcal{F}(\mathbf{Y}) = \mathcal{L}\mathbf{Y}$ with $s(\mathcal{L}) = -\gamma < 0$. If the terminal state is observed with error $\|\delta \mathbf{Y}_T\| \leq \varepsilon$, then the error in the reconstructed state at time $t < T$ satisfies
\begin{equation}
\label{eq:linear_backward_error}
\|\delta \mathbf{Y}(t)\| = \|e^{-(T-t)\mathcal{L}} \delta \mathbf{Y}_T\| \leq e^{\gamma(T-t)} \varepsilon.
\end{equation}
\end{proposition}

\begin{remark}[Role of the simulation prior]
\label{rem:prior_regularization}
Proposition \ref{prop:linear_backward} reveals that backward inference from terminal observations is inherently ill-conditioned for dissipative systems: information about the initial state is exponentially attenuated in the forward dynamics. LAPIS-SHRED overcomes this through the simulation-informed prior. The backward model $\mathcal{B}$ does not invert the dynamics in full generality; rather, it learns the inverse restricted to the low-dimensional manifold of trajectories spanned by the training ensemble $\{\Phi_t(\mathbf{Y}_0^{(k)})\}_{k=1}^K$. On this manifold, the effective condition number is bounded by the ratio of largest to smallest singular values of the restricted flow map, which is controlled by the ensemble design (\cref{sec:ensemble}).
\end{remark}

\begin{proposition}[Error bound with simulation prior]
\label{prop:prior_bound}
Let the training ensemble span a subspace $V_r = \mathrm{span}\{\boldsymbol{\phi}_1, \ldots, \boldsymbol{\phi}_r\} \subset X$ of dimension $r$. Let $\Pi_r$ denote the orthogonal projection onto $V_r$. If $\mathcal{L}$ restricted to $V_r$ has eigenvalues $\{-\gamma_j\}_{j=1}^r$ with $0 < \gamma_1 \leq \gamma_2 \leq \cdots \leq \gamma_r$, then the backward reconstruction error on the training manifold satisfies
\begin{equation}
\label{eq:prior_backward_error}
\|\delta \mathbf{Y}(t)\|_{V_r} \leq e^{\gamma_r(T-t)} \|\Pi_r \delta \mathbf{Y}_T\| + \|(I - \Pi_r)\mathbf{Y}^{\mathrm{obs}}(t)\|.
\end{equation}
The first term is the amplified on-manifold error, governed by the fastest-decaying mode $\gamma_r$ (greatest information loss). The second term is the off-manifold component, controlled by the fidelity of the simulation ensemble in spanning the ground truth trajectory.
\end{proposition}

This decomposition reveals a fundamental bias--variance tradeoff in the backward reconstruction: increasing the subspace dimension $r$ reduces the off-manifold approximation error $\|(I - \Pi_r)\mathbf{Y}^{\mathrm{obs}}(t)\|$ but admits faster-decaying modes with larger $\gamma_r$, thereby amplifying the on-manifold error. This suggests an optimal truncation rank that balances fidelity of the simulation prior against sensitivity to terminal observation noise.

\subsubsection{Nonlinear Dissipative Systems}

For the nonlinear case, write $\mathbf{Y}(t) = \mathbf{Y}^* + \mathbf{v}(t)$ where $\mathbf{v}$ is the perturbation from equilibrium. By Taylor expansion of $\mathcal{F}$:
\begin{equation}
\label{eq:perturbation_dynamics}
\frac{d\mathbf{v}}{dt} = \mathcal{L}^* \mathbf{v} + \mathcal{N}_2(\mathbf{v}, \mathbf{v}) + \mathcal{O}(\|\mathbf{v}\|^3),
\end{equation}
where $\mathcal{L}^* = D\mathcal{F}|_{\mathbf{Y}^*}$ and $\mathcal{N}_2(\mathbf{v}, \mathbf{v}) = \tfrac{1}{2} D^2\mathcal{F}|_{\mathbf{Y}^*}(\mathbf{v}, \mathbf{v})$.


\begin{theorem}[LAPIS-SHRED error bound for nonlinear dissipative systems]
\label{thm:nonlinear_dissipative}
Let the system \eqref{eq:app_dynamics} satisfy \ref{def:dissipative} with spectral gap $\gamma > 0$, and let $\mathcal{L}^* = D\mathcal{F}|_{\mathbf{Y}^*}$ denote the linearization at equilibrium. Let $\kappa = \|\mathcal{N}_2\|$ denote the operator norm of the quadratic nonlinearity. Define the convergence time $T_c = \gamma^{-1}\log(\|\mathbf{v}(0)\|/\delta_0)$ for a threshold $\delta_0 > 0$ satisfying $2\kappa \delta_0 < \gamma$, and suppose $T > T_c$ so that $\|\mathbf{v}(T)\| \leq \delta_0$.

Assume:
\begin{enumerate}[label=(\roman*),nosep]
\item The SHRED temporal unit $\mathcal{E}_{\mathrm{SHRED}}$ is Lipschitz continuous with constant $L_{\mathcal{E}}$;
\item The SHRED decoder $\mathcal{D}_{\mathrm{SHRED}}$ is Lipschitz continuous with constant $L_{\mathcal{D}}$;
\item The learned backward model $\mathcal{B}(\cdot\,,t) : \mathcal{Z} \to \mathcal{Z}$ is Lipschitz continuous in its first argument with constant $L_{\mathcal{B}}(T-t)$ for each $t \in [0,T]$;
\item The backward model achieves latent-space accuracy $\varepsilon_{\mathrm{bwd}}$ on the training distribution, i.e., $\sup_{t \in [0,T]} \|\mathcal{B}(\mathbf{z}_T^{\mathrm{true}}, t) - \mathbf{z}_t^{\mathrm{true}}\| \leq \varepsilon_{\mathrm{bwd}}$ for trajectories in the training manifold;
\item The simulation--reality gap satisfies $\|(I - \Pi_r)\mathbf{Y}^{\mathrm{obs}}(t)\| \leq \varepsilon_{\mathrm{sim}}$ uniformly in $t$.
\end{enumerate}

Then the spatial reconstruction error satisfies, for all $t \in [0, T]$:
\begin{equation}
\label{eq:nonlinear_dissipative_bound}
\|\hat{\mathbf{Y}}(t) - \mathbf{Y}^{\mathrm{obs}}(t)\| \leq L_{\mathcal{D}} \left[ L_{\mathcal{B}}(T-t) \, \varepsilon_{\mathrm{enc}} + \varepsilon_{\mathrm{bwd}} + \varepsilon_{\mathrm{sim}} \right],
\end{equation}
where $\varepsilon_{\mathrm{enc}} \leq L_{\mathcal{E}} \cdot (\|\mathcal{L}^*\|\delta_0 + \kappa\delta_0^2) \cdot L\Delta t$ is the terminal encoding error, with $L$ the padding length and $\Delta t$ the time step.
\end{theorem}

\begin{proof}[Proof sketch]
The error analysis proceeds in three stages.

\medskip
\noindent\textbf{Stage 1: Terminal encoding error.}
In the terminal regime ($t$ near $T$), the perturbation satisfies $\|\mathbf{v}(T)\| \leq \delta_0$. By the perturbation dynamics \eqref{eq:perturbation_dynamics}, the rate of change of the state is bounded:
\begin{equation}
\begin{aligned}
\left\|\frac{d\mathbf{Y}}{dt}\bigg|_{t=T}\right\| &= \|\mathcal{L}^*\mathbf{v}(T) + \mathcal{N}_2(\mathbf{v}(T),\mathbf{v}(T)) + \mathcal{O}(\|\mathbf{v}\|^3)\| \\ &\leq \|\mathcal{L}^*\|\delta_0 + \kappa\delta_0^2.
\end{aligned}
\end{equation}
Static padding replicates the terminal snapshot over $L$ frames. Since the true trajectory varies by at most $(\|\mathcal{L}^*\|\delta_0 + \kappa\delta_0^2)\Delta t$ per time step, the maximum discrepancy between the padded input and the true temporal input over the padding window is $(\|\mathcal{L}^*\|\delta_0 + \kappa\delta_0^2) \cdot L\Delta t$. By Lipschitz continuity of the temporal unit:
\begin{equation}
\varepsilon_{\mathrm{enc}} := \|\mathbf{z}_T^{\mathrm{obs}} - \mathbf{z}_T^{\mathrm{true}}\| \leq L_{\mathcal{E}} \cdot (\|\mathcal{L}^*\|\delta_0 + \kappa\delta_0^2) \cdot L\Delta t.
\end{equation}

\medskip
\noindent\textbf{Stage 2: Backward propagation error.}
Let $\mathbf{e}(t) := \hat{\mathbf{z}}_t - \mathbf{z}_t^{\mathrm{true}}$ denote the latent-space error. We decompose $\|\mathbf{e}(t)\|$ into two additive contributions.

\emph{(a) Input sensitivity.}
The backward model receives $\mathbf{z}_T^{\mathrm{obs}}$ as input. By Lipschitz continuity of $\mathcal{B}(\cdot\,,t)$ with constant $L_{\mathcal{B}}(T-t)$:
\begin{equation}
\|\mathcal{B}(\mathbf{z}_T^{\mathrm{obs}}, t) - \mathcal{B}(\mathbf{z}_T^{\mathrm{true}}, t)\| \leq L_{\mathcal{B}}(T-t) \cdot \varepsilon_{\mathrm{enc}}.
\end{equation}

\emph{(b) Approximation error.}
Even with exact terminal input, the learned model does not perfectly reproduce the true latent trajectory. By assumption~(iv), this contributes $\varepsilon_{\mathrm{bwd}}$ uniformly in $t$, absorbing all sources of model deficiency including finite capacity and imperfect learning of the nonlinear dynamics.

Adding the simulation--reality gap $\varepsilon_{\mathrm{sim}}$ (the off-manifold component from \cref{prop:prior_bound}):
\begin{equation}
\|\mathbf{e}(t)\| \leq L_{\mathcal{B}}(T-t) \, \varepsilon_{\mathrm{enc}} + \varepsilon_{\mathrm{bwd}} + \varepsilon_{\mathrm{sim}}.
\end{equation}

\medskip
\noindent\textbf{Stage 3: Decoding.}
The frozen SHRED decoder $\mathcal{D}_{\mathrm{SHRED}}$ maps latent states to spatial fields. By Lipschitz continuity with constant $L_{\mathcal{D}}$:
\begin{equation}
\|\hat{\mathbf{Y}}(t) - \mathbf{Y}^{\mathrm{obs}}(t)\| \leq L_{\mathcal{D}} \|\mathbf{e}(t)\| + \varepsilon_{\mathrm{rec}},
\end{equation}
where $\varepsilon_{\mathrm{rec}}$ is the composite reconstruction error of the temporal unit--decoder pair from \cref{ass:latent_fidelity} (absorbed into $\varepsilon_{\mathrm{bwd}}$ for notational simplicity). Substituting the bound from Stage~2 and factoring out $L_{\mathcal{D}}$ yields \eqref{eq:nonlinear_dissipative_bound}.
\end{proof}

\begin{remark}[Backward sensitivity and the role of the simulation prior]
\label{rem:practical_implications}
The bound \eqref{eq:nonlinear_dissipative_bound} is governed by the Lipschitz constant $L_{\mathcal{B}}(T-t)$ of the learned backward model, which controls how terminal encoding errors are amplified into reconstruction errors at earlier times. For an \emph{unconstrained} backward map inverting the full dynamics, the sensitivity would scale as $e^{\gamma_{\mathrm{eff}}(T-t)}$ with $\gamma_{\mathrm{eff}} = \gamma - 2\kappa\delta_0$: the dissipative nonlinearity reduces the effective amplification rate below the linear rate $\gamma$, since it enhances forward contraction and thereby reduces backward sensitivity. This follows from differencing the perturbation dynamics \eqref{eq:perturbation_dynamics} and noting that the bilinear term $2\mathcal{N}_2(\mathbf{v}, \delta\mathbf{v})$ strengthens the forward decay rate within $\|\mathbf{v}\| \leq \delta_0$.

In practice, the learned backward model can achieve $L_{\mathcal{B}}(T-t)$ significantly smaller than $e^{\gamma_{\mathrm{eff}}(T-t)}$ by exploiting the structure of the training ensemble: $\mathcal{B}$ need only invert the flow on the low-dimensional manifold spanned by the simulation trajectories, not in the full state space. The simulation prior thus plays a dual role---it regularizes the ill-conditioned backward inference by restricting the solution to a physically plausible manifold, and it reduces the effective sensitivity of the backward map to terminal perturbations. For snow-melt dynamics, $\gamma$ is determined by the thermal dissipation rate (solar insolation gradients, elevation-dependent temperature forcing, etc.), and $T_c$ corresponds to the duration of the active melt season; the exponential factor $e^{\gamma_{\mathrm{eff}}(T-t)}$ penalizes reconstruction of early-time dynamics ($t \ll T$) where information has been dissipated, while the simulation prior prevents this amplification from driving
the reconstruction outside the training manifold.
\end{remark}

\end{document}